\documentclass{article}
\usepackage{graphicx} 

\title{Generative neural networks for characteristic functions}
\author{Florian Brück }
\date{}

\usepackage{biblatex}
\usepackage{amsmath,amssymb,amsthm}
\usepackage{bm}
\usepackage{color}
\usepackage[ruled,vlined]{algorithm2e}
\usepackage{graphicx}
\usepackage{hyperref}
\usepackage{subfig,caption}
\usepackage{multirow}

\newcommand{\e}{\mathbb{E}}
\newcommand{\p}{\mathbb{P}}
\newcommand{\lk}{\left[ }
\newcommand{\rk}{\right] }
\newcommand{\lc}{\left(}
\newcommand{\rc}{\right)}
\newcommand{\R}{\mathbb{R}}
\newcommand{\C}{\mathbb{C}}
\newcommand{\N}{\mathbb{N}}

\newcommand{\id}{\mathbf{1}}

\newcommand{\sumn}{\sum_{i=1}^n}

\newcommand{\iinn}{{i\in\N}}
\newcommand{\leqn}{1\leq i\leq n}
\newcommand{\rmd}{\mathrm{d}}
\newcommand{\sumijn}{\sum_{\substack{i,j=1 \\ i\not=j} }^n}

\DeclareMathOperator{\argmin}{argmin}
\DeclareMathOperator{\mmd}{MMD}

\DeclareMathOperator{\sign}{sign}
\newcommand{\hmmd}{\widehat{\mmd^2}}

\newcommand{\bmX}{\bm{X}}
\newcommand{\bmx}{\bm{x}}
\newcommand{\bmz}{\bm{z}}
\newcommand{\bmY}{\bm{Y}}
\newcommand{\bmy}{\bm{y}}
\newcommand{\bmZ}{\bm{Z}}
\newcommand{\bmW}{\bm{W}}
\newcommand{\bmtheta}{{\bm{\theta}}}

\theoremstyle{plain}
\newtheorem{thm}{Theorem}[section]
\newtheorem{lem}{Lemma}
\newtheorem{rem}{Remark}

\newtheorem{cor}[thm]{Corollary}

\newtheorem{ass}{Assumption}

\allowdisplaybreaks

\begin{document}

\maketitle

\begin{abstract}
We provide a simulation algorithm to simulate from a (multivariate) characteristic function, which is only accessible in a black-box format. The method is based on a generative neural network, whose loss function exploits a specific representation of the Maximum-Mean-Discrepancy metric to directly incorporate the targeted characteristic function. The algorithm is universal in the sense that it is independent of the dimension and that it does not require any assumptions on the given characteristic function. Furthermore, finite sample guarantees on the approximation quality in terms of the Maximum-Mean Discrepancy metric are derived. The method is illustrated in a simulation study.
\end{abstract}

KEYWORDS: Characteristic function, generative modeling, simulation algorithm

\section{Introduction}

The characteristic function is one of the fundamental objects in probability theory, since it uniquely characterizes the distribution of a real-valued random vector in a concise way. Its properties often allow to simplify theoretical derivations, especially when sums of independent random variables are investigated. Further, it also allows to easily derive certain properties of the underlying random vector, such as its moments. A disadvantage of working with characteristic functions is that simulation from the corresponding random vector is not straightforward when there is no further information about the underlying random vector. As Devroye comments on the simulation from a (univariate) characteristic function in \cite{devroye2006}: \textit{``If the characteristic function is known in black-box format, very little can be done in a universal manner''}. This poses major challenges in applications, since simulation from the corresponding random vector is often essential to asses certain quantities of interest. 

Several approaches to simulate from a random vector that corresponds to a given characteristic function seem to naturally come to mind. There are various ways of ``inverting'' the characteristic function to obtain its corresponding (Lebesgue) density or distribution function, such as the Fourier inversion formula, L\'evy's characterization theorem and several other variants thereof. However, even though the theory provides clear constructive ways of recovering a corresponding density or distribution function, it is still a major challenge to follow them in practice. The main reason for this difficulty is that all these ways of recovering the corresponding density or distribution function require the evaluation of many integrals. More specifically, for every point at which the density or distribution function should be recovered, one integral has to be evaluated. 
Usually, one needs to resort to numerical integration routines to evaluate these integrals, which become computationally intensive in already moderate dimensions, since they require an exponential growth of function evaluations to keep the error stable across different dimensions \cite[Chapter IX.2]{asmussenglynn2007}. 
After recovering the density or distribution function on a grid, one has to extend these objects to globally valid densities or distribution functions. Assuming that this non-trivial task can be achieved, simulation from a given density or distribution function without further information on the underlying distribution can only be considered rather simple in the univariate case. For example, \cite{devroye1986,devroye1981,abatewhitt1992,barabesi2015} consider the univariate case and essentially apply Fourier inversion to obtain the density and a dominating density of the corresponding random variable to apply acceptance-rejection simulation techniques. Similarly, \cite{glassermanliu2010,chenfenglin2012} use Fourier inversion techniques to obtain the distribution function and density corresponding to a univariate characteristic function along with error guarantees on Monte Carlo estimates from the approximated distribution function. An extension of these approaches to the multivariate setting seems challenging and the literature on this problem is scarce. Often, the authors seem to focus on the bivariate case and on inverting the related, but numerically more convenient, Laplace transform, not the characteristic function, see e.g.\ \cite{choudrylucantoniwhitt1994,abatechoudrywhitt1998,abatewhitt2006,cappellowalker2018}. Further, neither of these works provides a corresponding simulation algorithm, which leaves this practically relevant aspect unanswered.

Instead of ``inverting'' the characteristic function to an object suitable for simulation, this paper proposes a novel machine learning based simulation algorithm that directly incorporates the given characteristic function into its loss function. The algorithm is inspired by generative machine learning models. Such models usually take an input dataset and try to generate samples that are as close as possible to the input data, where ``closeness'' is measured in terms of a certain ``distance criterion''.
In contrast to these scenarios, we are not provided with an input dataset that we would like to imitate, but we would like to learn a specific target distribution that is solely available in terms of its characteristic function. To achieve this, we use a generative neural network, whose loss function is based on a specific representation of the Maximum-Mean-Discrepancy metric (MMD) derived from a translation invariant kernel \cite{sriperumbudur2010hilbert}. The advantage of this choice of loss function is that it allows to exploit a representation of the MMD as a weighted $L_2$ distance of characteristic functions to directly incorporate the characteristic function into the loss function of the model, without the need for samples from the targeted distribution. Furthermore, an evaluation of the loss function only requires the ability to evaluate the given characteristic function at every argument, which allows to consider characteristic functions which are only accessible in a ``black-box'' format. The contribution of this paper can be summarized as follows.\\

\textbf{Contribution:} We provide a universal machine learning based method for simulation from a characteristic function, which is only assumed to be given in a black-box format. We solely require the ability to evaluate the characteristic function at every argument. Moreover, the suggested method is independent of the dimension of the underlying random vector, which makes multivariate applications as easy as univariate applications.\\

Potential applications of our framework can be found in the realm of Lévy processes, i.e.\ cádlág processes with independent and stationary increments, which are ubiquitous in applied sciences, see e.g.\ the monographs by \cite{sato, conttankov2003financial,applebaum2009} for an overview about the topic and its applications. Each Lévy process can be identified with an infinitely divisible distribution, which corresponds to the distribution of the increments of the process over a fixed time horizon. The famous Lévy-Khintchine representation of an infinitely divisible random vector provides that its characteristic function can be described in terms of a triplet that corresponds to a deterministic part, a Gaussian part and a Poissonian part. The non-trivial component of the characteristic function of an infinitely divisible distribution correspond to the Poissonian part, which is usually described in terms of an infinite measure called the Lévy measure. Most properties of a Lévy process are governed by the Lévy measure, which is therefore usually the object which is modeled in applications. However, one of the major practical challenges is to simulate from the corresponding infinitely divisible random vector, since usually the only accessible quantity describing its distribution is its characteristic function. Exact simulation schemes are rarely known and one usually needs to rely on approximate simulation schemes, which however suffer from the curse of dimensionality as they become quite slow in higher dimensions. Here, our simulation algorithm might provide a viable alternative, since after training the generative neural network, generating samples from the underlying distribution is as fast as evaluating a neural network.

The paper is organized as follows. Section \ref{secconstrgen} describes the construction of a generative neural network that generates samples from a given characteristic function. Section \ref{sectheorguarantees} provides theoretical guarantees on the approximation quality of such a generative neural network in terms of the MMD metric. Section \ref{secsimstud} illustrates the algorithm in a simulation study and Section \ref{secconclusion} summarizes the results and sketches open research questions. Proofs can be found in Appendix \ref{appproofs}.

\section{Construction of the generator}
\label{secconstrgen}

Assume that we are given the characteristic function of a probability measure $P$ on $\R^d$, i.e.\
\begin{align*}
\Phi_P: \R^d \to\C ;\ \bmz\mapsto \Phi_P(\bmz)=\e\lk e^{i\bmz^\intercal \bmX} \rk =\int_{\R^d} e^{i\bmz^\intercal \bmx} P(\rmd\bmx),
\end{align*}
Recall that we assume that $\Phi_P(\cdot)$ is the only information about $\bmX \sim P$ that we have access to. The goal of this section is to construct a generative neural network, called generator, that can (approximately) create samples from the probability measure $P$.

The general idea can be described as follows: Choose your favorite neural network architecture $N_\bmtheta: \R^{d^\prime} \to \R^d$, where $\bmtheta\in\Theta:=\R^{p}$ denotes the parameter vector of a neural network with $p$ parameters. Feed the neural network $N_\bmtheta(\cdot)$ i.i.d.\ random vectors $(\bmZ_i)_{\leqn}$ from a distribution $P_{\bmZ}$ which can be easily sampled. Take the output vectors $\lc \bmY_i\rc_{\leqn}$ of the neural network and ``compare`` how close they are to the target distribution $P$ via a suitable loss function, where we use the notation 
$$\bmY_i:=\bmY_i(\bmtheta) :=N_\bmtheta(\bmZ_i).$$ 
Update the neural network parameters $\bmtheta$ by taking a stochastic gradient descent step towards the minimum of the loss function. Iterate this procedure until convergence of $\bmtheta$ to the minimizer $\bmtheta^\star$ of the loss function is achieved. The resulting neural network $N_{\bmtheta^\star}$ should then approximately satisfy $N_{\bmtheta^\star}(\bmZ)\sim P$. 

This procedure is very well known in Machine Learning under the term generative modeling. The key difference here is that we are not given a sample dataset $(\bmX_i)_{i\in\N}$ from $P$ which can be used during training of the model. Instead, we target the theoretical representation of $P$ in terms of $\Phi_P$ directly in the loss function of our model. 

\subsection{MMD metrics based on translation invariant kernels}
Our choice for the loss function is based on specific representatives of the MMD metrics \cite{sriperumbudur2010hilbert}, namely those which are constructed from a translation invariant kernel. Before introducing our loss function explicitly, let us recall these special representatives of the $\mmd$ metrics. In general, every symmetric positive definite function $k:\R^d\times \R^d\to\R$, called kernel, defines a semi-metric on the space of probability measures on $\R^d$ by 
$$\mmd_k(P_1,P_2):=\lc \e\lk k(\bmX,\bmX^\prime) \rk -\e\lk 
k(\bmX,\bmY^\prime) \rk-\e\lk  k(\bmX^\prime,\bmY) \rk + \e\lk k(\bmY,\bmY^\prime ) \rk\rc^{1/2},$$
where the random vectors $\bmX,\bmX^\prime\sim P_1$ and $\bmY,\bmY^\prime\sim P_2$ are mutually independent. If $k(\bmx,\bmy)=\psi_k(\bmx-\bmy)$ for some continuous and positive definite function $\psi_k:\R^d \to\R$, $k$ is called translation invariant kernel. Moreover, when $\psi_k(\bm 0)=1$, Bochner's theorem implies that $k$ has the representation 
$$k(\bmx ,\bm y)=\e\lk \exp\lc i(\bmx-\bmy)^\intercal \bmW \rc\rk ,$$
where $\bmW$ denotes a unique symmetric random vector with values in $\R^d$. 
Exploiting this representation of $k$, \cite[Corollary 4]{sriperumbudur2010hilbert} shows that $\mmd_k(P_1,P_2)$ can be expressed as 
\begin{align}
     \mmd_k(P_1,P_2)=\lc  \e \lk \vert \Phi_{P_1}(\bmW)-\Phi_{P_2}(\bmW)\vert^2 \rk \rc^{1/2} ,\label{mmdascharfctdiff}
\end{align}
where $\vert z\vert $ denotes the modulus of a complex number $z$.
Therefore, when $k$ is a translation invariant kernel with $k(\bm 0,\bm 0)=1$, $\mmd_k(P_1,P_2)$ can be interpreted as a the expected distance of the characteristic functions $\Phi_{P_1}$ and $\Phi_{P_2}$ at the random location $\bmW$. Further, if the support\footnote{The smallest closed set $A$ s.t.\ $\p(\bmW\in A)=1$.} of $\bmW$ is equal to $\R^d$, \cite[Theorem 9]{sriperumbudur2010hilbert} shows that $\mmd_k(P_1,P_2)$ defines a proper metric on the space of probability measures on $\R^d$. Thus, from now on, we will make the following assumption, which is satisfied for most of the commonly used translation invariant kernels.

\begin{ass}
\label{ass1}
    The kernel $k$ satisfies $k(\bmx,\bmy)=\e\lk \exp\lc i\bmW^\intercal(\bmx- \bmy)\rc\rk$ for some random vector $\bmW$ with support $\R^d$.
\end{ass}

\begin{rem}
It is rather easy to find a kernel $k$ for which $\bmW$ is known and has support $\R^d$. For example, when $k(\bmx,\bmy)=\exp\lc \Vert \bmx-\bmy\Vert^2_2/\sigma\rc$, the random vector $\bmW$ has independent components which follow a Gaussian distribution with mean $0$ and variance $2/\sigma$. When $k(\bmx,\bmy)=\exp\lc \Vert \bmx-\bmy\Vert_1/\sigma\rc$, the random vector $\bmW$ has independent components which follow a Cauchy distribution with location parameter $0$ and scale parameter $1/\sigma$. In both cases, $\bmW$ has support $\R^d$. The kernels are called the Gaussian and Laplace kernel with bandwidth parameter $\sigma$, respectively.
\end{rem}

\subsection{Construction of the loss function}

We would like to construct a loss function which measures the distance of the distribution $P_\bmtheta$ of $N_\bmtheta(\bmZ)$ and our target probability distribution $P$, which is solely represented in terms of $\Phi_P$. However, $P_\bmtheta$ is usually inaccessible and therefore we have to resort to empirical approximations of $P_\bmtheta$. Here, this is done in terms of sampling i.i.d.\ observations $(\bmY_i)_{\leqn}=(N_\bmtheta(\bmZ_i))_{\leqn}$ from $P_\bmtheta$ and approximating $P_\bmtheta$ by its empirical measure $P_{\bmtheta,n}:=n^{-1}\sumn \delta_{\bmY_i}$. By virtue of (\ref{mmdascharfctdiff}), this allows to estimate $\mmd_k(P_\bmtheta,P)$ via
$$ \mmd_k(P_{\bmtheta,n},P) =\lc \e_{\bmW}\lk \Big\vert n^{-1} \sum_{i=1}^n \exp(i\bmW^\intercal \bmY_i) -\Phi_P(\bmW) \Big\vert^2 \rk\rc^{1/2}.$$  
Simple calculations show that
\begin{align*}
    \mmd_k(P_{\bmtheta,n},P)^2&=\frac{1}{n^2}\e_{\bmW}\lk  \sum_{i,j=1}^n \exp(i\bmW^\intercal (\bmY_i-\bmY_j)) \rk -\frac{2}{n} \e_{\bmW}\lk \sum_{i=1}^n \exp(-i\bmW^\intercal \bmY_i)\Phi_P(\bmW)\rk\\
    &+C_P,
\end{align*}
where $C_P:=\e_{\bmW}\lk \Phi_P(\bmW)\overline{\Phi_P(\bmW)} \rk$ is a constant that solely depends on $P$. Note that $\e_{\bmW}\lk \sum_{i=1}^n \exp(-i\bmW^\intercal \bmY_i)\Phi_P(\bmW)\rk$ is real-valued, since it can be expressed as $n\e_{\bmY\sim P_{\bmtheta,n}, \bmX\sim P} $ $\lk k(\bmX,\bmY)\rk$.
Usually, it is not realistic to assume that $\e_{\bmW}\Big[ \sum_{i=1}^n \exp(-i\bmW^\intercal \bmY_i)\Phi_P(\bmW)\Big] $ can be computed in closed form. Thus, we further approximate $\mmd_k(P_\bmtheta,P)$ using the approximations
\begin{align}
    \e_{\bmW}\lk \sum_{i=1}^n \exp(-i\bmW^\intercal \bmY_i)\Phi_P(\bmW)\rk \approx  \frac{1}{m} \sumn \sum_{l=1}^m \Re\lc \exp(-i\bmW_l^\intercal \bmY_i)\Phi_P(\bmW_l) \rc, \label{approxcrossterm}
\end{align}
and
\begin{align} \frac{1}{n^2}\e_{\bmW}\Big[ \sum_{i,j=1}^n  \exp(i\bmW^\intercal (\bmY_i-\bmY_j)) \Big] \approx \frac{1}{mn(n-1)} \sumijn \sum_{l=1}^m  \exp(i\bmW_l^\intercal (\bmY_i-\bmY_j))  ,\label{approxgeneratorterm}
\end{align}
where $\lc \bmW_l\rc_{1\leq l\leq m}$ denote i.i.d.\ copies of $\bmW$ and $\Re(z)$ denotes the real part of a complex number $z$. Note that we do not take into account the observations with indices $i=j$ in (\ref{approxgeneratorterm}), since they would introduce a bias in the estimate of $\mmd_k(P_\bmtheta,P)$. Further, it is necessary to only consider the real part of the right hand side of (\ref{approxcrossterm}), since, even though $\e_{\bmW}\lk \sum_{i=1}^n \exp(-i\bmW^\intercal \bmY_i)\Phi_P(\bmW)\rk$ is real-valued, $m^{-1}\sum_{l=1}^m \sumn \exp(-i\bmW_l^\intercal \bmY_i)\Phi_P(\bmW_l)$ does not need to be real-valued anymore. Taking the real part in (\ref{approxgeneratorterm}) is not necessary, since every term is summed with its complex conjugate, which gives a real-valued approximation.

Combining the approximations above, this allows to define our loss function for the training of $N_\bmtheta(\cdot)$ as
\begin{align}
L(\bmtheta)&:=L \lc \lc \bmY_i\rc_{1\leq i\leq n},\lc \bmW_l\rc_{1\leq l\leq m},\Phi_P \rc& \nonumber\\
&:=\frac{1}{mn(n-1)}\sumijn \sum_{l=1}^m    \exp(i\bmW_l^\intercal (\bmY_i-\bmY_j)) \nonumber \\
&-2\Re \lc\frac{1}{nm}\sum_{l=1}^m \sumn \exp(-i\bmW_l^\intercal Y_i)\Phi_P(\bmW_l)\rc \label{deflossfunction},
\end{align}
which is an approximation of the unknown quantity $\mmd_k(P_\bmtheta,P)^2-C_P$. Obviously, a $\bmtheta$ which minimizes $L\lc\lc \bmY_i\rc_{1\leq i\leq n},\lc \bmW_l\rc_{1\leq l\leq m},\Phi_P\rc$ is an approximation of $\argmin_{\bmtheta\in\Theta}\mmd_k$ $(P_\bmtheta,P)$. The quality of the approximation of the loss function is discussed in further detail in the next section.

Based on the loss function in (\ref{deflossfunction}), we may construct a generator of the probability distribution $P$ as follows: Choose $\bmZ\sim P_{\bmZ}$, $k$ (resp.\ $\bmW$) and a neural network $\lc N_\bmtheta(\cdot)\rc_{\bmtheta\in\Theta}$. Define the loss function $L(\bmtheta)$ as in (\ref{deflossfunction}) and find $\bmtheta^\star\in\argmin_{\bmtheta\in\Theta}L(\bmtheta)$. Then, $N_{\bmtheta^\star}(\bmZ)$ should approximately be distributed according to $P$. The pseudo-code of the algorithm is schematically summarized in Algorithm \ref{alg1}.

\SetKwInput{KWParam}{Requires}
\begin{algorithm}
\SetAlgoLined \LinesNumbered
\KwIn{Characteristic function $\Phi_P$ of target probability distribution $P$}
\KWParam{Kernel $k$, generator of $\bmW$/$\bmZ$, number of epochs $e$, and batch sizes $n,m\in\N$}
Initialize a neural network $N_\bmtheta:\R^{d^\prime}\to\R^d$ \;
\For{$1\leq k\leq e$}{
    Simulate i.i.d.\ random vectors $\lc\bmW_l\rc_{1\leq l\leq m}$ and $\lc\bmZ_i\rc_{1\leq i\leq n}$\;
    Compute $\lc\bmY_i\rc_{\leqn}=\lc N_\bmtheta(\bmZ_i)\rc_{\leqn}$\;
    Calculate $L(\bmtheta)=L\lc \lc\bmY_i\rc_{\leqn},\lc\bmW_l\rc_{1\leq l\leq m},\Phi_P\rc$\;
    Update $\bmtheta$ by taking a gradient step towards the minimizer of $L(\bmtheta)$. 
}
 \KwResult{Generator $N_\bmtheta(\cdot)$ of the probability distribution $P$.}
 \caption{Learning the generator of the probability distribution $P$ which is solely parameterized in terms of $\Phi_P$} 
 \label{alg1}
\end{algorithm}

\begin{rem}
    An alternative loss function could be defined via 
    \begin{align}
       \int_{\R^d} \vert \hat{\Phi}_n(\bmz)-\Phi_P(\bmz)\vert^2 w(\bmz)\rmd \bmz, \label{empfcharfctloss} 
    \end{align}\
    where $\hat{\Phi}_n(\bmz)=n^{-1}\sumn \exp(i\bmz^\intercal\bmY_i)$ denotes the empirical characteristic function of the sample $ \lc\bmY_i\rc_{1\leq i\leq n}$ and $w:\R^d\to [0,\infty)$ denotes a non-negative Lebesgue-integrable ``weighting'' function.
    A loss function of the form (\ref{empfcharfctloss}) has been used for estimation purposes and goodness-of-fit testing, see \cite{yu2004,meintanis2016} for reviews of the topic. It has also been used in generative machine learning to learn the distribution of a given dataset \cite{ansari2020characteristic,liyu2020,lili2023}. However, it has not been used for simulation purposes yet. A simple calculation shows that a simulation approach based on (\ref{deflossfunction}) is essentially equivalent to a simulation approach based on a discretized version of (\ref{empfcharfctloss}) when we assume that $\bmW$ has a density w.r.t.\ the Lebesgue measure or, equivalently, that $w(\cdot)$ integrates to finite value. 
\end{rem}

\begin{rem}
In the definition of our loss function (\ref{deflossfunction}) we approximate $\e_{\bmW} \big[ 
\exp( $ $ i\bmW^\intercal (\bmY_i-\bmY_j)) \big]$ according to (\ref{approxgeneratorterm}). When the kernel $k$ corresponding to $\bmW$ is known, this is not necessary, since we know that $\e_{\bmW} \lk 
\exp\lc i\bmW^\intercal (\bmY_i-\bmY_j)\rc \rk=k(\bmY_i,\bmY_j)$. However, even in these cases, we refrain from using the exact expression, since, due to the equivalence of our loss function with a discretized version of (\ref{empfcharfctloss}), one can easily see that the loss function in (\ref{deflossfunction}) defines a pseudo-metric for every sample $\lc \bmW_l\rc_{1\leq l\leq m}$ (up to a multiplication of one term with $(n-1)/n$). This property would be lost if we used the exact mathematical expression instead. Thus, to essentially work with a pseudo-metric in every step of the algorithm, we  use the approximation (\ref{approxgeneratorterm}) in our loss function, even when the corresponding kernel is known.
\end{rem}

\begin{rem}
\label{remconstant}
   Since the loss function $L\lc \lc \bmY_i\rc_{1\leq i\leq n},\lc \bmW_l\rc_{1\leq l\leq m},\Phi_P\rc$ is an approximation of $\mmd_k(P_\bmtheta,P)^2-C_P$, its realizations cannot be interpreted as a ``large'' or ``small'' loss in absolute terms. They may only be compared relatively to each other. However, $C_P=\e_{\bmW}\lk \Phi_P(\bmW)\overline{\Phi_P(\bmW)} \rk$ can be approximated by $m^{-1}\sum_{1\leq l\leq m}\Phi_P(\bmW_l)\Phi_P(-\bmW_l)=:\hat{C}_P$. Thus, to estimate $\mmd_k(P_\bmtheta,P)^2$ directly, one can use $L( \lc \bmY_i\rc_{1\leq i\leq n},\lc \bmW_l\rc_{1\leq l\leq m},\Phi_P )+\hat{C}_P$, which then allows to interpret the loss of network as ``large'' or ``small''. The term $\hat{C}_P$ could obviously be included in every calculation of the loss, but to speed up the computation of the loss we have refrained from doing so.
\end{rem}

\section{Theoretical guarantees}
\label{sectheorguarantees}
In general, providing a formal recipe of a neural network architecture that is reasonable for a problem at hand and comes with theoretical guarantees is known to be a difficult task. It is still an active area of research to provide theoretical guarantees for the approximation quality of neural networks with random input and general loss functions, see e.g.\ \cite{baileytelgarsky2018,lee2017,lulu2020}. In this work, the goal is to choose a suitable architecture of the neural network which is compatible with the loss function in (\ref{deflossfunction}), in the sense that it allows to approximately generate samples of an arbitrary probability distribution $P$.

Two approaches might be considered to obtain results about the approximation quality. First, one could try to quantify the quality of the approximation of a function $G:\R^{d^\prime}\to\R^d$ which satisfies $G(\bmZ)\sim P$. When $P_{\bmZ}$ is absolutely continuous w.r.t.\ the Lebesgue measure, the existence of such a function $G$ is ensured by Rosenblatt's transform \cite{rosenblatt1952}. Since the Rosenblatt transform is known to not be unique, it is clear that there are many functions $G$ which satisfy $G(\bmZ)\sim P$. Thus, it could be the case that $N_\bmtheta$ and a chosen $G$ are very different, even though $N_\bmtheta(\bmZ)$ is still a good approximation of a generator of $P$. Furthermore, it happens frequently that neither representative of $G$ obeys any regularity conditions such as continuity or differentiability and $G\in L_1(P_{\bmZ})$ if and only if $\e_{\bmX\sim P}\lk \vert X\vert \rk$ is finite. Thus, an analysis based on a specific representation of $G$ seems to be difficult.

On the other hand, a neater approach is to directly target the approximation in terms of the distance of the distribution of $N_\bmtheta(\bmZ)$ and $P$. In general, one would hope for an approximation result in terms of a metric that metrizes weak convergence of probability measures on $\R^d$. In our framework, this distance is naturally given by the $\mmd_k$ metric and an analysis of the approximation capabilities in terms of the $\mmd_k$ metric of a feedforward neural network with ReLu activation function has been recently conducted in \cite{yangliwang2022}. Their results can be summarized as follows: Consider a bounded, translation invariant and characteristic kernel $\vert k\vert\leq 1$ on $\R^d$. Then, if $P_{\bmZ}$ is an absolutely continuous probability distribution w.r.t.\ the Lebesgue measure,  there exists a fully connected feedforward neural network $N_{\bmtheta^\star}:\ \R^{d^\prime}\to\R^d$ with ReLu activation function, depth $h\geq 2$ and widths $(w_i)_{1\leq i\leq h}\geq 7d+1$ such that $\mmd_k(P_{\bmtheta^\star},P)\leq 160 \sqrt{d} \lc\max_{1\leq i\leq h} w_i \rc^{-1} h^{-1/2} $.

Since we assume that our kernel $k$ is bounded, translation invariant and continuous, \cite[Theorem 3.2]{sriperuimbudur2016} implies that the considered $\mmd_k$ metrics metrize weak convergence, i.e.\ $\mmd_k(P_n,P)\to 0$ if and only if $\lc P_n\rc_{\iinn}$ weakly converges to $P$. 
An immediate corollary is that there exists a sequence of neural networks whose distribution weakly converges to the target distribution.
\begin{cor}
\label{corapproxmmd}
    Let $\bmX$ denote an arbitrary random vector. Assume that $P_{\bmZ}$ is an absolutely continuous probability distribution w.r.t.\ the Lebesgue measure. Then there exists a sequence of fully connected neural networks $N_{\bmtheta_n}(\cdot)$ of depth $2$ with ReLu activation function such that $ N_{\bmtheta_n}(\bmZ)\to \bmX$ in distribution.
\end{cor}
Thus, the right architecture of a neural network allows to build a generator which is ``close'' to the target distribution not only in the $\mmd_k$ metric, but also in terms of the topology of weak convergence of probability measures. 

\cite{yangliwang2022} in combination with Corollary \ref{corapproxmmd} explains how to choose the hyperparameters of a fully connected feedforward neural network with ReLu activation function such that there exists a specific parametrization of the network, which is a good generator for $P$. However, to this point, it is not clear that our approximation of the loss function in (\ref{deflossfunction}) allows to find such an approximation and whether the impact of the additional hyperparameters $n$ and $m$ can be quantified.

In the following, we will show that it is possible to obtain explicit bounds (in probability) on the approximation quality of a neural network with the loss function in (\ref{deflossfunction}). To this purpose, we need to introduce some notation. Let us denote 
\begin{align*}
    \bmtheta_{n,m}&\in\argmin_{\bmtheta\in\Theta} L\lc (\bmY_i)_{1\leq i\leq n},(\bmW_l)_{1\leq l\leq m},\Phi_P)\rc. 
\end{align*}
To ensure that $\theta_{n,m}$ is well-defined we need to make an additional technical assumption, similar to \cite[Assumption 1]{briol2019statistical}.

\begin{ass}
\label{assexmin}
    \begin{enumerate}
   Let $k_{m}(\bmx,\bmy):=\frac{1}{2m}\lc \sum_{l=1}^m \exp(i \bmW_l(\bmy-\bmx) + \exp(-i \bmW_l(\bmy-\bmx) \rc$ denote a random translation invariant and bounded kernel. Conditionally on almost every realization of $\lc \bmW_l\rc_{1\leq l\leq m}$ we assume
        \item For every probability measure $P$ there exists a $c>0$ such that $ \{ \bmtheta\in\Theta\mid \mmd_{k_m}(P_\bmtheta,P)$ $\leq \inf_{\bmtheta\in\Theta}\mmd_{k_m}(P_\bmtheta,P) +c\}$ and $ \{ \bmtheta\in\Theta\mid \mmd_{k}(P_\bmtheta,P)\leq \inf_{\bmtheta\in\Theta}\mmd_{k}(P_\bmtheta,P) +c\}$ are bounded.
        \item For every $n\in\N$ and probability measure $P$ there exists a $(c_n)_{n\in\N}>0$ such that $ \{ \bmtheta\in\Theta\mid \mmd_{k_m}(P_{\bmtheta,n},P)\leq \inf_{\bmtheta\in\Theta}\mmd_{k_m}(P_{\bmtheta},P) +c_n \}$ is bounded almost surely.
    \end{enumerate}
\end{ass}
Further, let $\Vert f\Vert_{P^{\bmW}_m}:=\sqrt{ \frac{1}{m}\sum_{l=1}^m f(\bmW_l)^2}$ and denote $\mathcal{F}_\Theta:=\{ f_\bmtheta \mid \bmtheta\in\Theta\}$, where  
\begin{align*}
          f_\theta:\R^d\to [0,\infty);\   f_\bmtheta(\bm w)&:=\vert \Phi_P(w)-\Phi_{P_\bmtheta}(w) \vert^2.
\end{align*} 
Finally, denote the empirical covering number w.r.t.\ $\Vert \cdot\Vert_{P^{\bmW}_m}$ of a class of functions $\mathcal{F}$ acting on $\bmW$ as $N\lc \epsilon,\mathcal{F},\Vert \cdot\Vert_{P^{\bmW}_m} \rc$. Essentially, covering numbers quantify the complexity of a class of functions and are commonly used in machine learning to express the complexity of a class of models. For example, for the well-known VC-classes of functions the empirical covering numbers can be bounded by a polynomial function of $\epsilon$ which is independent of $m$. For more details on covering numbers we refer to \cite{wellner2013weak}. 

With the notation at hand, we are ready to state explicit bounds on the approximation quality of a fully connected feedforward neural network that only depends on hyperparameters that can be chosen by the user.
\begin{thm}
\label{thmtheoreticalguarantees}
     Assume that $k$ satisfies Assumptions \ref{ass1} and \ref{assexmin}, $P_{\bmZ}$ is an absolutely continuous probability distribution w.r.t.\ the Lebesgue measure, $N_\bmtheta(\cdot)$ has ReLu activation function, $h\geq 2$, $(w_i)_{1\leq i\leq h}\geq 7d+1$ and assume that $\e\lk \int_0^{4} \sqrt{ \log \lc  N\lc \epsilon,\mathcal{F}_\Theta, \Vert\cdot\Vert_{P_m^{\bmW}} \rc\rc  } \rmd \epsilon  \rk\leq f(m)$ for some function $f$ which satisfies $f(m) m^{-1/2}\to 0$. Then for every $\tau>0$ we have
        \begin{align*}
            \mmd_k(P_{\bmtheta_{n,m}},P)&\leq  160 \sqrt{d} \lc\max_{1\leq i\leq h} w_i \rc^{-1} h^{-1/2}  + 2\sqrt{\frac{24f(m)+8\lc 1+\sqrt{8\log(2/\tau)}\rc }{\sqrt{m}}} \\
        &\ \ +  2 \sqrt{\frac{2}{n}}\lc 2+\log\lc\frac{2}{\tau}\rc \rc.
        \end{align*}
        with probability at least $1-2\tau$.
\end{thm}
Thus, if we assume that the complexity of $\lc P_\bmtheta\rc_{\bmtheta\in\Theta}$ is not too large and that our optimization procedure is able to find $\bmtheta_{n,m}$, we can choose parameters $n,m,h,(w_i)_{1\leq i\leq h}$ such that $P_{\bmtheta_{n,m}}$ is arbitrarily close to $P$ in terms of the $\mmd_k$ metric with high probability. In other words, it is reasonable to assume that a small empirical loss in (\ref{deflossfunction}) means that the law of $N_{\bmtheta_{n,m}}(\bmZ)$ is close to $P$ in the $\mmd_k$ metric. Moreover, as the $\mmd_k$ metric metrizes weak convergence, we can also assume that the law of $N_{\bmtheta_{n,m}}(\bmZ)$ is close to $P$ in the topology of weak convergence of probability measures.

\section{Simulation results}
\label{secsimstud}

This section illustrates our proposed simulation algorithm for a random vector corresponding to a given characteristic function. The purpose is to illustrate that the method works in a rather universal manner, which is why we have chosen the same hyperparameters for all experiments that we conducted. The code that was used to train the models and generate the plots can be found on \url{https://github.com/Flo771994/charfctgen/}.

According to the results of Section \ref{sectheorguarantees}, we chose a standard feedforward neural network architecture with ReLU activation function and two hidden layers. The input layer has width $2d$, the first hidden layer has width $300$, the second hidden layer has width $50$ and the activation function for the last layer is chosen as a linear mapping to a $d$-dimensional output vector. The input random vector $\bmZ$ was chosen to have independent standard normal distributed margins. For all our experiments we used a convex combination of Gaussian kernels with bandwidths $b\in\{ 0.02,0.5,1,5,100\}=:B$, which leads to the corresponding random variable 
$$ \bmW\sim  \sqrt{2/\eta}\tilde{\bmW},$$ 
where $\eta$ is uniformly distributed on $B$ and $\tilde{\bmW}$ follows a multivariate normal distribution with independent standard normal distributed margins. The very small and large bandwidths $0.02$ and $100$ were necessary to prevent the neural network from learning a large constant in some cases. This can occur since $\mmd_k(c,P)$ is approximately constant for large $c$, which then mistakenly can be considered as a local minimum of the loss function.
The batch sizes $n,m$ have been fixed to $6000$ as well as the number of epochs $e$. As a little modification to Algorithm \ref{alg1}, we have kept the sample $(\bmW_l)_{1\leq l\leq m}$ constant across $20$ epochs to avoid resampling in every iteration. To update the parameters by a gradient descent step, we chose to use the automatic differentiation routines of PyTorch \cite{pytorch} in combination with the ADAM optimizer introduced in \cite{kingmaba2015}. We started with an initial learning rate $0.01$ and adapted the learning rate after $2000$ and $4000$ epochs by multiplying the current learning rate with $0.1$.

We have conducted two experiments. As a first experiment, we chose a Gaussian mixture model with $J$ mixture components, where each of the $J$ multivariate Gaussian distributions is chosen with equal probability. The corresponding characteristic function is given by
$$ \Phi^{(1)}_{\bar{\mu},\bar{\Sigma}}(\bmz)=\frac{1}{J}\sum_{j=1}^J\exp\lc i\mu_j^\intercal\bmz -\bmz^\intercal \Sigma_j\bmz/2 \rc,$$
where $\bar{\mu}_k=(\mu_1,\ldots \mu_J) \in\R^{d\times J}$ and $\bar{\Sigma}=(\Sigma_1,\ldots,\Sigma_J)\in \R^{(d\times d)\times J}$
This example is particularly interesting, because the resulting distribution is multimodal, which is usually considered as a challenging property to learn in machine learning. Further, the density and distribution function are still tractable, which allows for comparison with the ground-truth and exact simulation schemes. We generated the random vectors corresponding to the characteristic function $\Phi^{(1)}_{\bar{\mu},\bar{\Sigma}}$ in dimensions $d=2,5,10$ and for $J=1,2,5,10$ mixture components. The location parameters $\bar{\mu}$ and dependence parameters $\bar{\Sigma}$ were generated randomly using sklearn's \textit{spd\_matrix} function in Python. 

We report plots of the estimated marginal densities for a sample from the generator. Further, we report bivariate contour plots of the estimated bivariate densities of the last two margins of the generator and from the exact simulation algorithms that are available in the numpy package in Python. The estimates were obtained using the \textit{gaussian\_kde} function from the Scipy package in Python. We have additionally plotted the true marginal densities and contour levels of the underlying distributions. The number of samples used to generate these plots was set to $10^6$, where we have fixed the contour levels to $\{0.0005,0.001,0.0025,0.005,0.01,0.05,0.1\}$ across all examples. Figure \ref{figgausmix} contains representative plots for each of the experiments. Additional plots can be found in the corresponding GitHub repository and in Appendix \ref{appaddplots}. 

\begin{figure}
\centering
    \subfloat[Gaussian mixture distribution $2$-dim with $2$ mixture components]{
    \includegraphics[scale=0.45]{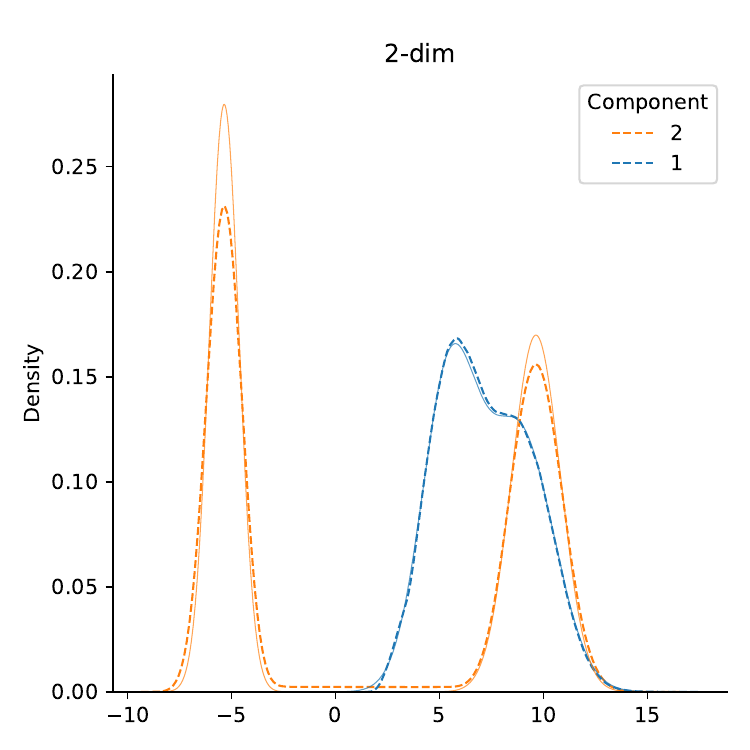}
    \includegraphics[scale=0.45]{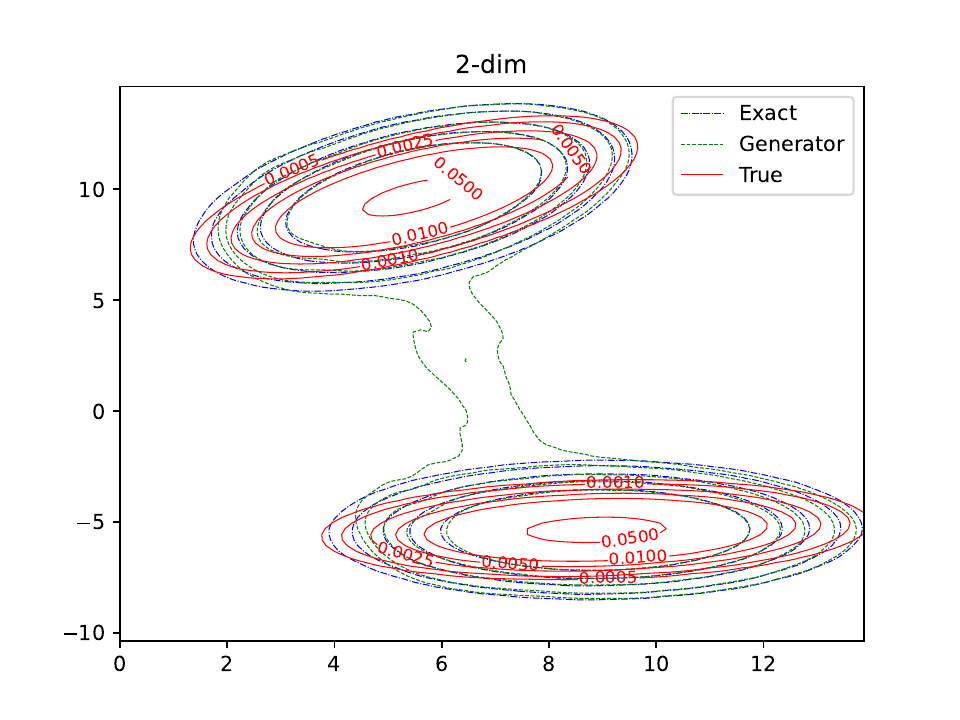}  
    }\vfill  
    \subfloat[Gaussian mixture distribution $5$-dim with $5$ mixture components]{
        \includegraphics[scale=0.45]{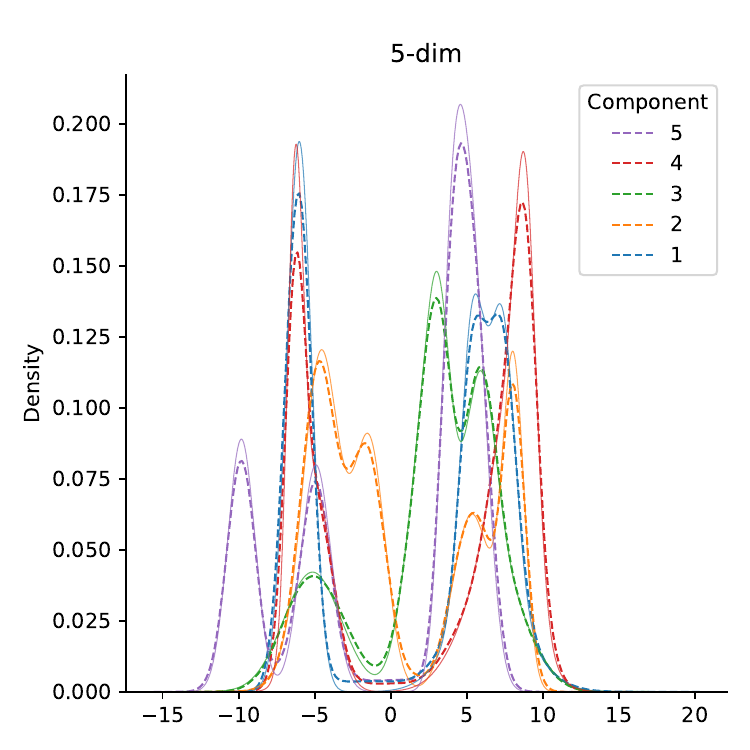}
    \includegraphics[scale=0.45]{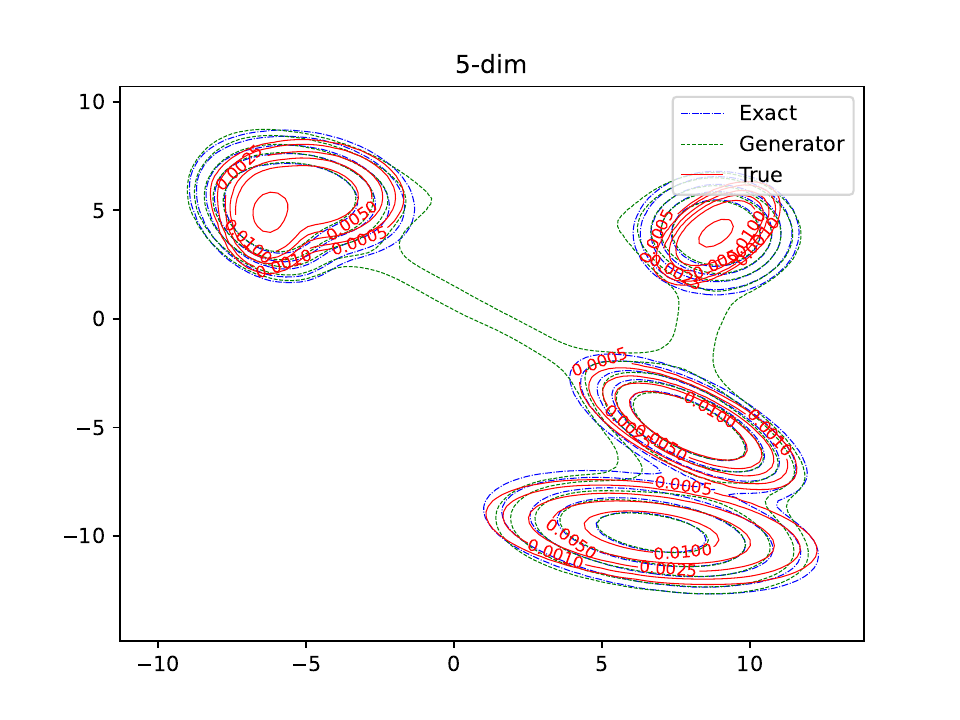}  
    }\vfill  
     \subfloat[Gaussian mixture distribution $10$-dim with $10$ mixture components]{
        \includegraphics[scale=0.45]{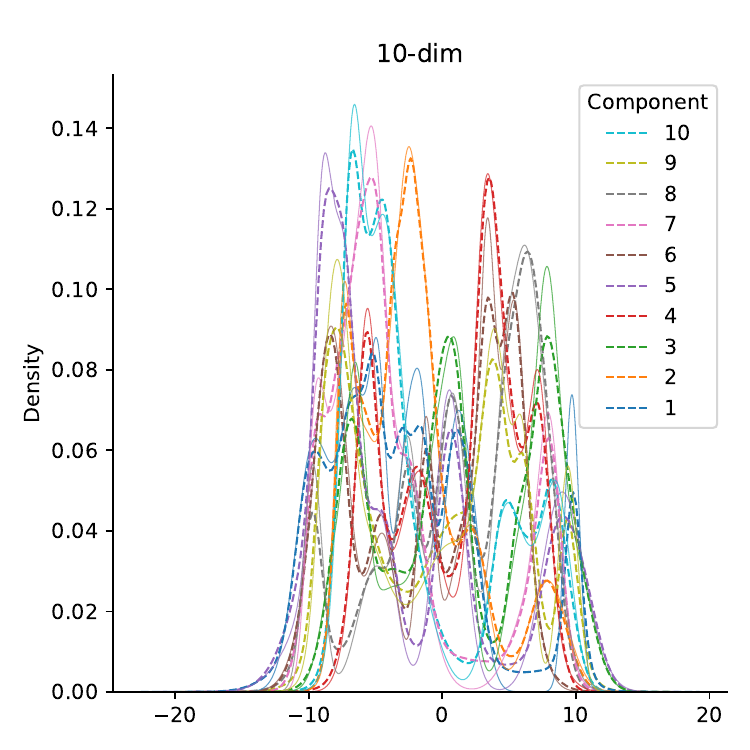}
    \includegraphics[scale=0.45]{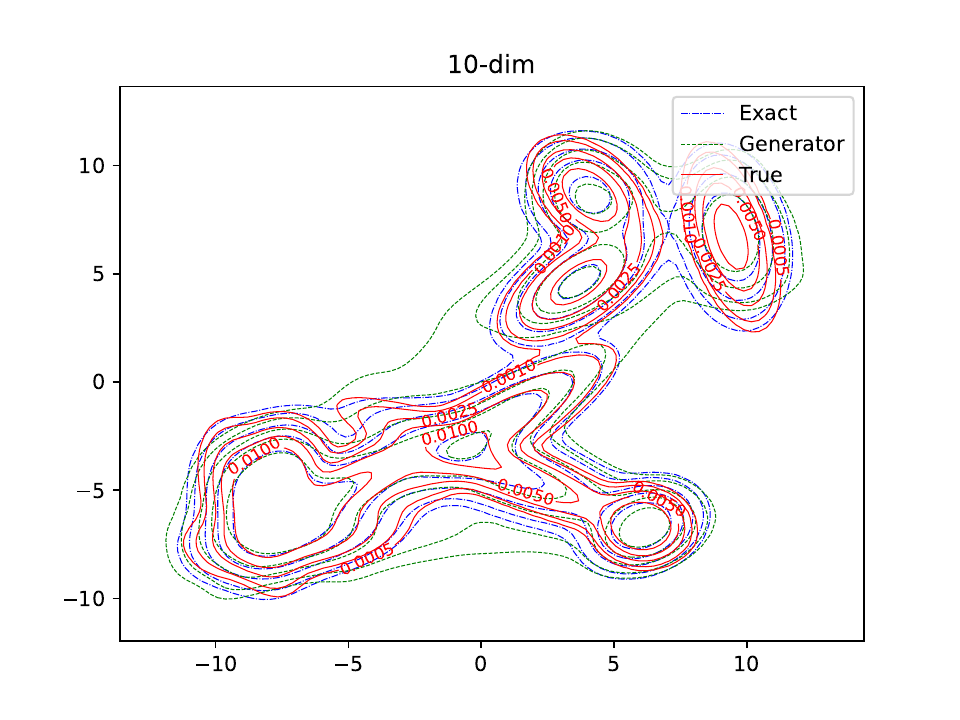}  
    }\vfill 
    \caption{Estimated marginal (left) and bivariate (right) densities of the Gaussian mixture distribution in dimensions $2,5,10$ with $2,5,10$ mixture components. The bivariate densities correspond to the last two components of the corresponding random vector. The solid red lines show the true densities and the dashed green and dash-dotted blue lines correspond to estimated densities from the generator and the exact simulation algorithm.}
\label{figgausmix}
\end{figure}

A first remarkable observation is that the simulation algorithm is able to find all the univariate and bivariate modes of the Gaussian mixture distribution, independently of the dimension and number of mixture components. The marginal and bivariate contour plots suggest that the modes are not as strictly separated as one would expect from an exact sample of the target distribution, even though the bivariate contour plots of the exact samples also do not perfectly resemble the theoretical contours\footnote{ This indicates that the number of simulated random vector is too low, as the estimated densities from the exact simulation algorithm should recover the true contour lines for a large enough sample size. Increasing the number of simulated random vectors from $10^6$ to $10^7$ only slightly improved the results, but further increasing the number of simulated random vectors was not possible under the constraint of conducting the computations in a reasonable amount of time on the available hardware.}. A comparison of Figure \ref{figgausmix} with Figure \ref{figgaussadd} in the Appendix shows that the performance seems to slightly decrease with an increasing number of mixture components. This is expected, since the complexity of the learned distributions increases with an increasing number of mixture components. On the other hand, the performance seems to be less dependent on the dimension of the learned distribution. As a summary of this example one can draw the conclusion that the simulation algorithm successfully learns the main features of the targeted distributions, but, unsurprisingly, an exact simulation algorithm is to be preferred when it is available.

As a second experiment, we chose to simulate from a characteristic function of an $\alpha$-stable distribution, which is one of the most popular infinitely divisible distributions used to define an associated Lévy process. Such models exhibit heavy tails as the corresponding random vectors only have moments up to order $\alpha$. The corresponding characteristic function is given by
\begin{align}
\Phi^{(2)}_{\alpha,\Lambda,\bm\tau}(\bmz)=\exp\lc i\bmz^\intercal \bm \tau -\int_{S_{d-1}} f_\alpha(\bmz,\bm u) \Lambda(\rmd \bm u)  \rc,  \label{eqncharfctstabledist}
\end{align}
where $\bm\tau\in\R^d$ is a shift parameter,
$$ f_\alpha(\bmz,\bm u)=\begin{cases}
   \vert \bmz^\intercal \bm u \vert^{\alpha}\lc 1-i\tan\lc 
\frac{\pi\alpha}{2} \sign \lc \bmz^\intercal \bm u \rc \rc \rc & \alpha\not=1;\\
\vert \bm z^\intercal \bm u \vert+i\frac{2}{\pi} \lc\bm z^\intercal \bm u \rc \log\lc \vert \bm z^\intercal \bm u\vert \rc & \alpha=1 
\end{cases} $$
and $\Lambda$ is a finite measure on the unit sphere $S_{d-1}$ of the Euclidean norm on $\R^d$. In the univariate case, exact simulation schemes are available, even though closed form expression of the density and distribution function are not known. As pointed out in \cite{nolan2014}, simulation of a multivariate $\alpha$-stable random vector is challenging and essentially only a small subclass of $\alpha$-stable distributions may be exactly simulated in dimension $d\geq 2$. Thus, in practice, one usually resorts to an approximate simulation algorithm based on truncating the Lévy measure of the $\alpha$-stable distribution, which we will use as a benchmark. Details concerning the approximate simulation algorithm can found in \cite[Section 6.3]{conttankov2003financial}, but everything that is relevant to our simulation study is that the algorithm has a tuning parameter $\epsilon$ which governs the trade-off between run-time and accuracy of the algorithm. Our simulation algorithm may provide an alternative to the standard approximate simulation algorithm of $\alpha$-stable distributions, since the characteristic function is available in closed form. However, exact evaluation of the integral in (\ref{eqncharfctstabledist}) is difficult, which is why we w.l.o.g.\ assume that the total mass of $\Lambda$ is equal to $1$, since a scaling of $\Lambda$ can be translated into a scaling and translation of the corresponding random vector. This assumption allows to approximate $\int_{S_d} f_\alpha(\bmz,\bm u) \Lambda(\rmd \bm u)$ via a Monte-Carlo approximation by drawing samples from $\Lambda$ and setting $\int_{S_d} f_\alpha(\bmz,\bm u) \Lambda(\rmd \bm u)\approx M^{-1}\sum_{j=1}^M f_\alpha(\bmz,\bm u_j)$, where $M$ denotes the number of samples from $\Lambda$. In this example, we will assume that $\Lambda$ is the law of $\bm Z/\Vert \bm Z\Vert_2$, where $\bmZ$ follows a $d$-dimensional Gaussian distribution with mean $\mu$ and covariance matrix $\Sigma$. For our experiments we have chosen $\alpha \in \{1/2,1\} $, $\mu=\bm 0$ and set $\tau=\bm 1$. Further, we chose $M=6000$ and $\Sigma\in\{ \text{id}_d,\text{id}_d+\id_d/2\}$, where $\text{id}_d$ denotes the $d$-dimensional identity matrix and $\id_d$ denotes a $d$-dimensional matrix containing $1$ in every entry. To benchmark out algorithm we chose $\epsilon=0.1$, which roughly corresponds to a $10-100$ times slower runtime of the approximate simulation algorithm in comparison to sampling from the generator. A comparison of the two algorithms with similar runtime was not possible, since in this case $\epsilon$ would have to be chosen so large that the approximate simulation algorithm would often output the same fixed constant.

Due to the heavy tails of the simulated distributions, density estimation cannot be used as standard measure of comparison, since the kernel density estimates are heavily deformed by the (correct) outliers in the data. Therefore, we have chosen to compare the performance of the simulation algorithm in terms of estimated quantiles of projections of the form $\langle \bm u,\bm Y \rangle$, where $\bm u$ was fixed to the first $d$-components of the vectors $\bm u_1:=\bm 1$, $\bm u_2:=(1,-1,1,-1,1,-1,1,-1,1,-1)$, $\bm u_3:=(-1,2,-1,2,-1,2,-1,2,-1,2)$. We have opted for such a comparison since projections of $\alpha$-stable random vectors depend on the underlying dependence structure and it is known that they follow a univariate $\alpha$-stable distribution \cite[Chapter 1]{samorodnitskytaqqu1994} whose parameters may be derived from the measure $\Lambda$. We have chosen the quantiles $q=0.05,0.1,0.3,0.5,0.7,0.9,0.95$ across all examples. The following two tables exemplarily show the resulting estimated quantiles, whereas a full list of plots and tables of all the conducted experiments can be found in the corresponding GitHub repository. Similarly to the first experiment, the estimates are based on $10^6$ samples from the generator and the approximate simulation algorithm, respectively.

\begin{table}
\footnotesize
    \begin{tabular}{llllllrrr}
    \hline
           & Quantile    & 0.05   & 0.1    & 0.3   & 0.5   &   0.7 &   0.9 &   0.95 \\
    \hline
           & Generator   & -50.22 & -12.54 & 1.27  & 1.97  &  2.48 &  9.08 &  22.21 \\
     $\bm u_1$ & True        & -45.24 & -8.51  & 1.41  & 2.00  &  2.59 & 12.46 &  49.04 \\
           & Approximate & -44.70 & -8.58  & 1.41  & 2.00  &  2.58 & 12.54 &  49.24 \\
     \hline
           & Generator   & -31.10 & -10.06 & -0.68 & -0.04 &  0.52 &  9.33 &  29.36 \\
     $\bm u_2$ & True        & -47.06 & -10.46 & -0.59 & 0.00  &  0.59 & 10.51 &  47.27 \\
           & Approximate & -47.11 & -10.50 & -0.59 & 0.00  &  0.58 & 10.5  &  47.25 \\
     \hline
           & Generator   & -55.90 & -16.85 & 0.08  & 1.04  &  1.97 & 14.79 &  42.55 \\
     $\bm u_3$ & True        & -73.53 & -15.57 & 0.07  & 1.00  &  1.94 & 17.61 &  75.68 \\
           & Approximate & -73.46 & -15.62 & 0.06  & 1.00  &  1.92 & 17.58 &  75.59 \\
    \hline
    \end{tabular}
    \caption{Estimated quantiles of the projection of $\langle \bm u_i,\bmY\rangle$ for $\alpha=1/2$, $d=2$ and $\Sigma=\text{id}_d$.  }
    \label{tab05stab2dim}
\end{table}

\begin{table}
\footnotesize
    \begin{tabular}{llllllrrr}
    \hline
           & Quantile    & 0.05    & 0.1     & 0.3   & 0.5   &   0.7 &   0.9 &   0.95 \\
    \hline
           & Generator   & -427.47 & -211.99 & 6.85  & 9.71  & 10.27 & 13.69 &  18.57 \\
     $\bm u_1$ & True        & -68.69  & -7.50   & 9.01  & 10.00 & 10.98 & 27.45 &  88.48 \\
           & Approximate & -68.49  & -7.48   & 9.03  & 10.00 & 10.97 & 27.44 &  88.83 \\
     \hline
           & Generator   & -79.24  & -39.10  & -1.18 & -0.10 &  0.16 &  1.88 &   4.31 \\
     $\bm u_2$ & True        & -34.27  & -7.62   & -0.43 & 0.00  &  0.43 &  7.68 &  34.54 \\
           & Approximate & -34.22  & -7.67   & -0.42 & 0.00  &  0.42 &  7.66 &  34.63 \\
     \hline
           & Generator   & -104.76 & -49.69  & 3.07  & 4.83  &  5.33 &  9.04 &  14.41 \\
     $\bm u_3$ & True        & -60.02  & -9.46   & 4.18  & 5.00  &  5.81 & 19.41 &  69.81 \\
           & Approximate & -60.13  & -9.52   & 4.20  & 5.00  &  5.80  & 19.41 &  69.85 \\
    \hline
    \end{tabular}

    \caption{Estimated quantiles of the projection of $\langle \bm u_i,\bmY\rangle$ for $\alpha=1/2$, $d=10$ and $\Sigma=\text{id}_d+\id_d/2$.}
        \label{tab05stab10dim}
\end{table}

\begin{table}
\footnotesize
    \centering
\begin{tabular}{llllllrrr}
\hline
       & Quantile    & 0.05   & 0.1   & 0.3   & 0.5   &   0.7 &   0.9 &   0.95 \\
\hline
       & Generator   & -2.50  & 4.64  & 8.78  & 9.92  & 11.01 & 14.11 &  17.61 \\
 $\bm u_1$ & True        & 0.37   & 5.30  & 8.89  & 10.00 & 11.11 & 14.7  &  19.62 \\
       & Approximate & 0.39   & 5.36  & 8.95  & 10.00 & 11.05 & 14.65 &  19.59 \\
 \hline
       & Generator   & -3.77  & -2.04 & -0.50 & 0.00  &  0.51 &  2.26 &   5.04 \\
 $\bm u_2$ & True        & -4.46  & -2.18 & -0.51 & -0.00 &  0.51 &  2.18 &   4.46 \\
       & Approximate & -4.45  & -2.14 & -0.48 & -0.00 &  0.48 &  2.16 &   4.46 \\
 \hline
       & Generator   & -5.21  & 0.56  & 4.01  & 4.95  &  5.89 &  8.57 &  11.56 \\
 $\bm u_3$ & True        & -3.25  & 0.98  & 4.05  & 5.00  &  5.95 &  9.03 &  13.27 \\
       & Approximate & -3.23  & 1.01  & 4.11  & 5.00  &  5.89 &  8.97 &  13.21 \\
\hline
\end{tabular}
\caption{Estimated quantiles of the projection of $\langle \bm u_i,\bmY\rangle$ for $\alpha=1$, $d=10$ and $\Sigma=\text{id}_d+\id_d/2$.  }
    \label{tab1stab10im}
\end{table}
\normalsize

The first observation is that the bulk of the distribution is fitted reasonably well by the generator across all examples. However, for the very heavy tailed case $\alpha=1/2$, the tails of the distribution are not captured by the generator, whereas the approximate simulation algorithm fits all quantiles well, which is to be expected since, under a very mild condition, the tails of an infinitely divisible random vector are fully captured by the approximate simulation algorithm. Further, when the dimension increases the fit of the tails of the generator deteriorates as can be seen from comparing Tables \ref{tab05stab2dim} and \ref{tab05stab10dim}. This observation can also be theoretically explained, since it is rather easy to see that a ReLu neural network preserves the order of the tails of the input. Since we are using Gaussian random vectors as input for the ReLu neural network one cannot model the very heavy tails of $1/2$-stable distributions. For the less heavy tails of $1$-stable distributions we could observe a much better fit, which is almost competitive with the fit of the approximate simulation algorithm, as reflected in Table \ref{tab1stab10im}. Since, after training, the simulation from the generator is between $10-100$ faster than the approximate simulation algorithm it may even be considered advantageous to use the generator in scenarios where fast sampling from a $1$-stable distribution is required.
Summarizing, one can say that the generator again seems to pick up the main features of the distribution, but due to the false tails of the input random vector it cannot properly capture the heavy tails of the $\alpha$-stable distributions.

A common observation that was noticed during both experiments is that an adequate representation of the corresponding random vectors is only reached when the loss $L(\bmtheta) +\hat{C}_P$ is very small. In our experiments, we needed a loss smaller than $0.01$ to obtain adequate results, which is probably due to the fact that our visual comparison of the distributions in terms of marginal densities and contour plots requires that the generator generates samples that are close to the target distribution in the topology of weak convergence as well. It is reasonable to assume that this is the case when the loss is very small, since we know that our choice of $\mmd_k$ metrizes weak convergence. Moreover, we have estimated the $\mmd_k$ distance of the samples from the generator to the target distribution $P$ as well as to samples from the exact or approximate simulation algorithm, respectively. Usually, the absolute value of the estimated $\mmd_k$ values was smaller than $0.003$, indicating that all underlying distributions are very similar in terms of the $\mmd_k$ metric. We should also mention that, even though the loss seems to stabilize only after a few hundred iterations (see Figure \ref{figloss}), further training after this stabilization improved the results. This is probably due to the fact that every sample of $(\bmW_l)_{1\leq l\leq m}$ specifies a new pseudo-metric which is used for comparison and some meaningful pseudo-metrics to learn the distribution are only sampled occasionally.

\begin{figure}
    \centering
    \includegraphics[scale=0.50]{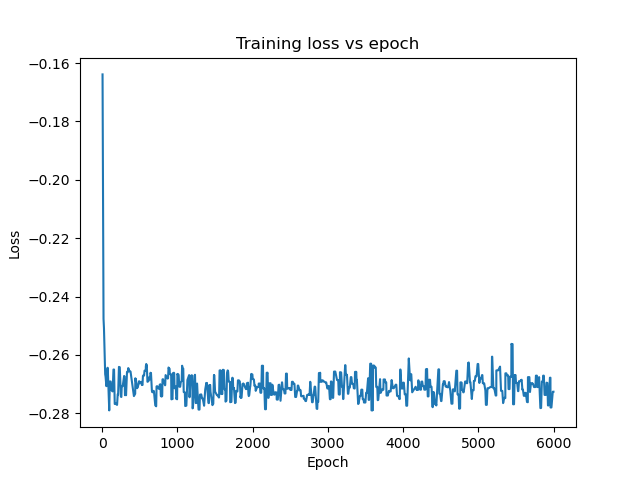}
    \caption{Exemplary visualization of the training loss during $6000$ epochs}
    \label{figloss}
\end{figure}

\section{Conclusion and Discussion}
\label{secconclusion}

\subsection{Conclusion}

We have provided a simulation algorithm for a random vector corresponding to a given characteristic function, which is only accessible in a black-box format. The algorithm comes with statistical guarantees in terms of the $\mmd_k$ metric. To the best of the authors knowledge, the proposed algorithm is the only algorithm for the simulation from a characteristic function than can be applied in every dimension without any assumption on the underlying characteristic function. Furthermore, we have demonstrated that the algorithm seems to work reasonably well in practice. The main observation is that key features of the distributions are reflected in the generated random vectors, even in higher dimensional examples, which suggest that the proposed simulation algorithm is a viable method for the simulation from a given characteristic function. 

\subsection{Discussion}
On the one hand, the universality of the algorithm is an attractive feature. On the other hand, there are many statistical problems where one knows much more about a random vector than its corresponding characteristic function, such as tail-behavior, moments or even that it belongs to a certain family of distributions. In such cases, the proposed simulation algorithm cannot be considered as the optimal solution, since the only adjustable parameter with significant influence on such quantities is the distribution of the input random vector $\bmZ$ and the activation function of the neural network. For example, when it is known that the corresponding random vector does not have a first moment, this feature cannot be modeled by a Neural Network with ReLu activation function and Gaussian random vectors $\bmZ$, since the order of the tails is preserved by any ReLu neural network. This theoretical observation is also confirmed by the simulation study, since the simulated approximations of $\alpha$-stable random vectors yield slimmer tails those of an $\alpha$-stable distributions. To appropriately model the tails of a random vector one would need to allow for activation functions that can adapt to certain tail behavior or use the ReLu activation function in combination with an input random vector that has the correct tail behavior. Unfortunately, replacing the ReLu activation function with a different activation function makes us lose the statistical guarantees and choosing an input random vectors with the correct tail behavior can only be achieved if one assumes certain knowledge of the corresponding random vector. Thus, it is an open research question if a more flexible network architecture can be found that is still universally applicable, comes with statistical guarantees and can adapt to certain stylized features of the target distribution.

Apart from the already mentioned example of Lévy processes, a potential application of our simulation algorithm can be found in Bayesian nonparametrics. Often, a random probability measure (prior) is constructed from a functional of a completely random measure and the posterior distribution is conjugate in the sense that it remains a deterministic transformation of a completely random measure, see e.g.\ \cite{doksum1974,jameslijoipruenster2009,lijoinipoti2014}. Posterior inference is usually conducted on the basis of MCMC methods, since even though the characteristic functional of the completely random measure is tractable, simulating from it is rather difficult. Often, the interest is not directly on the posterior distribution of the completely random measure, but in statistical functionals of the form $I_f=\int_{\R^d} f(x) \mu(\mathrm{d}x)$, where $\mu$ is the posterior completely random measure. Since completely random measures are infinitely divisible, $I_f$ is an infinitely divisible random variable. Moreover, its characteristic function is tractable, see \cite[Lemma 3]{brück2023}. Thus, our suggested simulation algorithm can also be applied to simulate from $I_f$. 

A potential extension of the framework to the simulation of parametric families of characteristic functions is left for future research. In principle, the parameter vector could be fed as an additional input parameter into the generator network. Then, slightly adapting the training procedure and the loss function should allow to learn how to simulate from a parametric family of characteristic functions. Such an extension would be valuable in practice, since it would potentially open the door to simulation of cádlág processes with independent non-stationary increments, called additive processes. They are a natural extension of Lévy processes, but their practical usefulness is often hampered by the complexity of the corresponding simulation algorithms, since each increment of the process corresponds to a different infinitely divisible distribution, which is usually a member of a certain parametric family distributions. However, obtaining statistical guarantees in this framework is challenging, which is why we have not considered this setting in this paper. Another open question is whether it is possible to obtain guarantees on the approximation quality of functionals of the distribution corresponding to a given characteristic function, such as expectations or quantiles. Theorem \ref{thmtheoreticalguarantees} only provides a bound on the $\mmd_k$ metric, which is equivalent to a bound on $\sup_{f\in \mathcal{H}_k} \vert \e_{\bmZ\sim P_{\bmZ}}\lk f(N_{\bmtheta_{n,m}}(\bmZ)) \rk-\e_{\bmX\sim P}\lk f(\bmX)\rk \vert$, where $\mathcal{H}_k$ denotes the unit ball of the reproducing kernel Hilbert space associated to the kernel $k$. However, even though $\mathcal{H}_k$ can be explicitly constructed from a given kernel $k$, it is not clear how useful these bounds are in practice, since many functionals of interest cannot be represented by elements of $\mathcal{H}_k$. A thorough analysis of this problem seems challenging and we leave such an analysis for future research.

\printbibliography

\newpage

\bigskip
\begin{center}
{\large\bf SUPPLEMENTARY MATERIAL}
\end{center}

\section{Proof of Theorem \ref{thmtheoreticalguarantees}}
\label{appproofs}

The proof of Theorem \ref{thmtheoreticalguarantees} is based on the following Lemma, whose proof is deferred to the end of this section.
\begin{lem}
\label{lemlossfctisrandommmd}
For a continuous translation invariant and bounded kernel $k(\bmx,\bmy)=\e[ \exp ( i \bmW^\intercal$ $(\bmx-\bmy) ) ]$ denote
    $$ \hmmd_k(P_{\bmtheta,n},P):=\frac{1}{n(n-1)}\sumijn k(\bmY_i,\bmY_j)-\frac{2}{n}\sumn \e_{\bmX\sim P}\lk k(\bmY_i,\bmX)\rk+\e_{\bmX,\bmX^\prime\sim P}\lk k(\bmX,\bmX^\prime) \rk  $$
    and define a random translation invariant and bounded kernel by $k_{m}(x,y):=\frac{1}{2m}\big( \sum_{l=1}^m $ $\exp(i \bmW_l(\bmy-\bmx) + \exp(-i \bmW_l(\bmy-\bmx) \big)$, where $\lc \bmW_l\rc_{1\leq l\leq m} \overset{i.i.d.}{\sim} \bmW$. Under the same assumptions as for Theorem \ref{thmtheoreticalguarantees} the following is true.
    \begin{enumerate}
        \item $L\lc (\bmY_i)_{1\leq i\leq n},(\bmW_l)_{1\leq l\leq m},\Phi_P)\rc =\hmmd_{k_{m}}(P_{\bmtheta,n},P)-C_P.$
        \item $\mmd_{k_m}(P_{\bmtheta_{n,m}},P) \leq \inf_{\bmtheta\in\Theta}\mmd_{k_m}(P_\bmtheta,P) +2 \sqrt{\frac{2}{n}}\lc 2+\log\lc\frac{2}{\tau}\rc \rc   $
        with probability at least $1-\tau$.
        \item Let $\Vert f\Vert_{P^{\bmW}_m}:=\sqrt{ \frac{1}{m}\sum_{l=1}^m f(\bmW_l)^2}$ and assume that $\e\lk \int_0^{4} \sqrt{ \log \lc  N\lc \epsilon,\mathcal{F}_\Theta,\Vert\cdot\Vert_{P_m^{\bmW}} \rc\rc  } \rmd \epsilon  \rk$ $\leq f(m)$ with $\mathcal{F}_\Theta=\{ f_\bmtheta \mid \bmtheta\in\Theta\}$ and  
        \begin{align*}
            f_\bmtheta(\bm w)&=\frac{1}{2}  \e_{\bmX,\bmX^\prime}\lk  \exp\lc i\bm w^\intercal(\bmX-\bmX^\prime) \rc+\exp\lc -i\bm w^\intercal(\bmX-\bmX^\prime) \rc  \rk\\
            &-\e_{\bmX,\bmY}\lk \exp\lc i\bm w^\intercal(\bmX-\bmY) \rc +\exp\lc -i\bm w^\intercal(\bmX-\bmY) \rc   \rk\\
            &+\frac{1}{2}\e_{\bmY,\bmY^\prime}\lk  \exp\lc i\bm w^\intercal(\bmY-\bmY^\prime) \rc+\exp\lc -i\bm w^\intercal(\bmY-\bmY^\prime) \rc  \rk ,
        \end{align*}
        where $\bmX,\bmX^\prime\sim P$ and $\bmY,\bmY^\prime\sim P_\bmtheta$ are mutually independent.
        Then, 
        $$\sup_{\bmtheta\in\Theta}\vert \mmd^2_k(P_\bmtheta,P)-\mmd^2_{k_m}(P_\bmtheta,P)\vert \leq \frac{24f(m)+8\lc 1+\sqrt{8\log(2/\tau)}\rc }{\sqrt{m}} $$
        with probability at least $1-\tau$.
    \end{enumerate}
    
\end{lem}

Using that $\vert\sqrt{x}-\sqrt{y}\vert\leq \sqrt{\vert x-y\vert}$ a consequence of 3.\ in Lemma \ref{lemlossfctisrandommmd} is that with probability at least $1-\tau$ we have
$$ \sup_{\bmtheta\in\Theta}\vert \mmd_{k_m}(P_\bmtheta,P)-\mmd_{k}(P_\bmtheta,P)\vert \leq \sqrt{\frac{24f(m)+8\lc 1+\sqrt{8\log(2/\tau)}\rc }{\sqrt{m}}}. $$

\begin{proof}[Proof of Theorem \ref{thmtheoreticalguarantees}]
    Observe that
    \begin{align*}
        \mmd_k(P_{\bmtheta_{n,m}},P)&= \mmd_k(P_{\bmtheta_{n,m}},P)-\mmd_{k_m}(P_{\bmtheta_{n,m}},P)+\mmd_{k_m}(P_{\bmtheta_{n,m}},P) \\
        &\leq \sup_{\bmtheta\in\Theta} \Big\vert \mmd_k(P_{\bmtheta},P)-\mmd_{k_m}(P_{\bmtheta},P) \Big\vert +\mmd_{k_m}(P_{\bmtheta_{n,m}},P) .
    \end{align*}
    Thus, Lemma \ref{lemlossfctisrandommmd} implies that with probability at least $1-2\tau$ we have
    \begin{align*}
        \mmd_k(P_{\bmtheta_{n,m}},P)&\leq \inf_{\bmtheta\in\Theta} \mmd_{k}(P_{\bmtheta},P) + 2\sqrt{\frac{24f(m)+8\lc 1+\sqrt{8\log(2/\tau)}\rc }{\sqrt{m}}} \\
        &\ \ +  2 \sqrt{\frac{2}{n}}\lc 2+\log\lc\frac{2}{\tau}\rc \rc.
    \end{align*}
    Since $\inf_{\bmtheta\in\Theta}\mmd_k(P_{\bmtheta},P)\leq  160 \sqrt{d} \lc\max_{1\leq i\leq h} w_i \rc^{-1} h^{-1/2}$ by \cite[Theorem 2.8]{yangliwang2022} the claim follows.
\end{proof}

\begin{proof}[Proof of Lemma \ref{lemlossfctisrandommmd}]
Observe that $k_{m}(\bmx,\bmy)=\e_{\Tilde{\bmW}}\lk \exp(i \Tilde{\bmW}(\bmx-\bmy)\rk$, where $\Tilde{\bmW}\sim \frac{1}{2m}\big( \sum_{l=1}^m$ $ \delta_{\bmW_l} +\sum_{l=1}^m \delta_{-\bmW_l} \big)$ denotes a symmetric random variable. Thus, $k(\cdot,\cdot)$ is symmetric, continuous and real-valued and by Bochner's theorem it is also positive definite, i.e.\ it is a kernel. Thus, $\mmd_{k_m}$ is well defined.
\begin{enumerate}
    \item  We obtain 
    \begin{align*}
        &\hmmd_{k_{m}}(P_{\bmtheta,n},P)^2-C_P =\frac{1}{n(n-1)}\sumijn k_m(\bmY_i,\bmY_j)-\frac{2}{n}\sumn \e_{\bmX}\lk k_m(\bmY_i,\bmX) \rk  \\
        &=\frac{1}{n(n-1)}\sumijn \frac{1}{2m} \sum_{l=1}^m  \exp\lc i\bmW_l^\intercal(\bmY_i-\bmY_j) \rc +\exp\lc -i\bmW_l^\intercal(\bmY_i-\bmY_j) \rc  \\
        &-\frac{2}{n}\sumn \e_{\bmX}\lk \frac{1}{2m} \sum_{l=1}^m  \exp\lc i\bmW_l^\intercal(\bmY_i-\bmX) \rc+\exp\lc -i\bmW_l^\intercal(\bmY_i-\bmX) \rc \rk \\
        &=\frac{1}{n(n-1)}\sumijn \frac{1}{2m} \sum_{l=1}^m 2\Re \lc  \exp\lc i\bmW_l^\intercal(\bmY_i-\bmY_j) \rc  \rc  \\
        &-\frac{2}{n}\sumn  \frac{1}{2m} \sum_{l=1}^m \e_{\bmX}\lk 2\Re \lc \exp\lc i\bmW_l^\intercal(\bmY_i-\bmX) \rc \rc \rk  \\
        &= \frac{1}{n(n-1)m}\sumijn  \sum_{l=1}^m   \exp\lc i\bmW_l^\intercal(\bmY_i-\bmY_j) \rc   \\
        &-\frac{2}{nm}\sumn   \sum_{l=1}^m  \Re \lc  \exp\lc i\bmW_l^\intercal\bmY_i \rc  \e_{\bmX}\lk \exp\lc -i\bmW_l^\intercal \bmX) \rc \rk \rc \\
        &=\frac{1}{n(n-1)m}\sumijn  \sum_{l=1}^m   \exp\lc i\bmW_l^\intercal(\bmY_i-\bmY_j) \rc    \\
        &-\frac{2}{nm} \Re \lc\sumn \sum_{l=1}^m      \exp\lc -i\bmW_l^\intercal\bmY_i \rc \Phi_P(\bmW_l) \rc \\
        &=L\lc (\bmY_i)_{1\leq i\leq n},(\bmW_l)_{1\leq l\leq m},\Phi_P)\rc
    \end{align*}
    
    \item 
        First, observe that $\bmtheta_{n,m}\in\argmin_{\bmtheta\in\Theta}\hmmd_{k_m}(P_{\bmtheta,n},P)$ by 1. Further, conditionally on $\lc \bmW_l\rc_{1\leq l\leq m}$, $k_m$ is deterministic and the conditional distribution of $\bmtheta_{n,m}$ given $\lc \bmW_l\rc_{1\leq l\leq m}$ is that of $\argmin_{\bmtheta\in\Theta}\hmmd_{k_m}(P_{\bmtheta,n},P)$, where $k_m$ is deterministic. Next, we remark that \cite{briol2019statistical} define an estimator 
        $$\hat{\bmtheta}_{k,n,T}\in\argmin_{\bmtheta\in\Theta} \hmmd_k(P_{\bmtheta,n},P_T)-C_{P_T}+\frac{1}{T(T-1)}\sum_{\underset{i\not=j}{i,j=1}}^{T}k(\bmX_i,\bmX_j), $$ where $P_T:=T^{-1}\sum_{i=1}^T \delta_{\bmX_i}$ denotes the empirical measure of an i.i.d.\ sample of size $T\in\N$ from $P$ and $k$ is a deterministic kernel. Moreover, \cite[Theorem 1]{briol2019statistical}\footnote{Note that the authors of \cite{briol2019statistical} assume that $k$ is characteristic, but an inspection of their proofs reveals that $k$ does not need to be characteristic to derive the generalization bounds.} provides the following generalization bounds on 
        $\mmd_k(P_{\hat{\bmtheta}_{k,n,T},n},P)$: For every $\tau>0$ one has
        $$  \mmd_k(P_{\hat{\bmtheta}_{k,n,T}},P) \leq \inf_{\bmtheta\in\Theta}\mmd_k(P_\bmtheta,P) +2\lc \sqrt{\frac{2}{n}}+\sqrt{\frac{2}{T}}\rc C(k)\lc 2+\log\lc\frac{2}{\tau}\rc \rc  $$
        with probability at least $1-\tau$, where $C(k):=\sqrt{ \sup_{\bmx\in\R^d}k(\bmx,\bmx)}$.
        
        The key step in the proof is to observe that there exists a subsequence $(T_h)_{h\in\N}$ such that we have $\lim_{h\to\infty}\hat{\bmtheta}_{k_m,n,T_h}\in\argmin_{\bmtheta\in\Theta}\hmmd_{k_m}(P_{\bmtheta,n},P)=:A$ almost surely when we restrict $\hat{\bmtheta}_{k_m,n,T}\in\argmin_{\bmtheta\in\Theta\cap A^\beta} \hmmd_{k_m}(P_{\bmtheta,n},P_T)-C_{P_T}+\frac{1}{T(T-1)}\sum_{\underset{i\not=j}{i,j=1}}^{T} $ $k(\bmX_i,\bmX_j)$ for some $\beta>0$ with $A^\beta:=\{ \bmtheta \in\R^p \mid d(\bmtheta,A)\leq \beta \}$ and $d$ denoting the Euclidean distance. This immediately implies that we also have $\mmd_{k_m}(P_{\bmtheta_{n,m}},P)=\lim_{h\to\infty} \mmd_{k_m}(P_{\hat{\bmtheta}_{k_m,n,T_h}},P)$, since $\bmtheta \mapsto \mmd_{k_m}(P_{\bmtheta},P)$ is continuous.
        
        To see this note that $\hmmd_{k_m} ( P_{\hat{\bmtheta}_{k_m,n,T},n},P_T)\leq \hmmd_{k_m} ( P_{\hat{\bmtheta}_{m,n},n},P_T)$ for all $T$, which implies $\limsup_{T\to\infty}\hmmd_{k_m} ( P_{\hat{\bmtheta}_{k_m,n,T},n},P_T)\leq \lim_{T\to\infty}\hmmd_{k_m} ( P_{\hat{\bmtheta}_{m,n},n},P_T)$ $=\hmmd_{k_m} (P_{\hat{\bmtheta}_{m,n},n},P)$. Now, if $\lim_{h\to\infty}\bmtheta_{k_m,n,T_h}\not\in \argmin_{\bmtheta\in\Theta} \hmmd_{k_m}(P_{\bmtheta,n},P)$ then, with positive probability, there exist an $\epsilon>0$ such that $\limsup_{h\to\infty}\hmmd_{k_m}$ $ ( P_{\hat{\bmtheta}_{k_m,n,T_h},n},P)>\hmmd_{k_m} ( P_{\hat{\bmtheta}_{m,n},n},P)+\epsilon$. However, 
        \begin{align*}
            &\Big\vert \hmmd_{k_m} ( P_{\hat{\bmtheta}_{k_m,n,T},n},P)-\hmmd_{k_m} ( P_{\hat{\bmtheta}_{k_m,n,T},n},P_T) \Big\vert \\
            &=\Big\vert C_P-C_{P_T}+\frac{2}{n}\sumn \lc \frac{1}{T} \sum_{j=1}^T k_m(\bmY_i,\bmX_j)-\e_{\bmX}\lk k_m(\bmY_i,\bmX) \rk \rc\Big\vert.
        \end{align*}
        The first term obviously converges to $0$. The second term converges to $0$ conditionally on $\lc \lc \bmZ_i\rc_{1\leq i\leq n},\lc \bmW_l\rc_{1\leq l\leq m}\rc$ (and thus conditionally on $\lc \bmW_l\rc_{1\leq l\leq m}$), because $n,m$ are fixed and the class of functions $\mathcal{G}:=\cup_{i=1}^n\{ \theta\mapsto k_m(N_\bmtheta(\bmZ_i),\cdot) \mid \bmtheta\in A^\beta \}$ has finite $L_1(P)$-covering number which implies that the Glivenko Cantelli theorem holds, i.e.\ $\sup_{\bmtheta\in A^\beta}\Big\vert \frac{1}{T}\sum_{J=1}^T  k_m(\bmY_i,\bmX_j) -\e_{\bmX}\lk k_m(\bmY_i,\bmX) \rk \Big\vert\to 0$. This is true since $k_m(\bmy,\cdot)$ is Lipschitz with a Lipschitz constant only depending on $\lc \bmW_l\rc_{1\leq l\leq m}$ and $N_\bmtheta(\bmZ_i)$ is Lipschitz in $\bmtheta$ with a Lipschitz constant that only depends on $\lc \beta, \lc \bmZ_i\rc_{1\leq i\leq n}\rc$. This implies that $\mathcal{G}$ is a class of functions which is Lipschitz on $\bmtheta$ and thus its $L_1(P)$ covering number is bounded in terms of the covering number of $A^\beta$ due to \cite[Section 2.7.4]{wellner2013weak}, conditionally on  $\lc \lc \bmZ_i\rc_{1\leq i\leq n},\lc \bmW_l\rc_{1\leq l\leq m}\rc$. Since $A^\beta$ is compact due to Assumption \ref{assexmin}, it has a finite covering number.  
        Thus, we get
            \begin{align*}
                \hmmd_{k_m} (P_{\hat{\bmtheta}_{m,n},n},P)&\geq \limsup_{T\to\infty}\hmmd_{k_m} ( P_{\hat{\bmtheta}_{k_m,n,T},n},P_T)\\
                &=\limsup_{T\to\infty}\hmmd_{k_m} ( P_{\hat{\bmtheta}_{k_m,n,T},n},P) >\hmmd_{k_m} ( P_{\hat{\bmtheta}_{m,n},n},P)+\epsilon,
            \end{align*}
            which is a contradiction. Therefore we have $\lim_{h\to\infty}\bmtheta_{k_m,n,T_h}\in \argmin_{\bmtheta\in\Theta} \hmmd_{k_m}$ $(P_{\bmtheta,n},P) $ for some subsequence $(T_h)_{h\in\N}$ since $A^\beta$ is compact.

        Combining the above, conditionally on $\lc W_l\rc_{1\leq l\leq m}$, $\bmtheta_{n,m}$ is distributed as $\lim_{h\to\infty}\hat{\bmtheta}_{k_m,n,T_h}$, where $k_m$ is a deterministic bounded and translation invariant kernel. Therefore, since $\beta$ was arbitrary we can choose $\beta\lc\lc W_l\rc_{1\leq l\leq m}\rc$ large enough such that $\inf_{\bmtheta\in\Theta}\mmd_{k_m}(P_\bmtheta,P) =\inf_{\bmtheta\in\Theta\cap A^{\epsilon}}$ $\mmd_{k_m}(P_\bmtheta,P)$ and we can apply the generalization bounds of \cite{briol2019statistical} to obtain
        \begin{align*}
            &\p\lc \mmd_{k_m}(P_{\bmtheta_{n,m}},P) \leq  \inf_{\bmtheta\in\Theta}\mmd_{k_m}(P_\bmtheta,P) +2 \sqrt{\frac{2}{n}} \lc 2+\log\lc\frac{2}{\tau}\rc \rc \rc\\
            &=\e\Bigg[ \p\Big( \mmd_{k_m}(P_{\bmtheta_{n,m}},P) - \inf_{\bmtheta\in\Theta}\mmd_{k_m}(P_\bmtheta,P) \\
            &=-2 \sqrt{\frac{2}{n}} \lc 2+\log\lc\frac{2}{\tau}\rc \rc  \leq 0\ \bigg\vert \ \lc \bmW_l\rc_{1\leq l\leq m}  \Big)\Bigg] \\
            &= \e \Bigg[ \p \Bigg( \lim_{h\to\infty} \mmd_{k_m}(P_{\hat{\bmtheta}_{k_m,n,T_h}},P) - \inf_{\bmtheta\in\Theta}\mmd_{k_m}(P_\bmtheta,P) \\
            &-2\lc \sqrt{\frac{2}{n}}+\sqrt{\frac{2}{T_h}}\rc \sqrt{ \sup_{\bmx\in\R^d}k_m(\bmx,\bmx)}  \lc 2+\log\lc\frac{2}{\tau}\rc \rc \leq 0\ \bigg\vert\ \lc \bmW_l\rc_{1\leq l\leq m} \Bigg)   \Bigg] \\
            &\geq \e\Bigg[ \liminf_{h\to\infty} \p\Bigg(  \mmd_{k_m}(P_{\hat{\bmtheta}_{k_m,n,T_h}},P) - \inf_{\bmtheta\in\Theta}\mmd_{k_m}(P_\bmtheta,P) 
            \\
            &-2\lc \sqrt{\frac{2}{n}}+\sqrt{\frac{2}{T_h}}\rc \sqrt{ \sup_{\bmx\in\R^d}k_m(\bmx,\bmx)} \lc 2+\log\lc\frac{2}{\tau}\rc \rc \leq 0\   \bigg\vert \lc \bmW_l\rc_{1\leq l\leq m} \Bigg)  \Bigg] \\
            &\overset{\star}{\geq} 1-\tau,
        \end{align*}
        where the regularity conditions for an application of \cite[Theorem 1]{briol2019statistical} in $\star$ are satisfied for almost every realization of $\lc \bmW_l\rc_{1\leq l\leq m}$ due to Assumption \ref{assexmin}. 
        \item First, $\e\lk \exp\lc i\bmW^\intercal (\bmx-\bmy)\rc\rk=\e\lk \exp\lc -i\bmW^\intercal (\bmx-\bmy)\rc\rk=k(\bmx,\bmy)$ implies $\e \lk k_m(\bmx,\bmy)\rk=k(\bmx,\bmy)$. Next, we observe that 
        \begin{align*}
            \mmd_{k_m}(P_\bmtheta,P)^2  &= \e_{\bmX,\bmX^\prime}\lk  k_m(\bmX,\bmX^\prime)  \rk-2\e_{\bmX,\bmY}\lk  k_m(\bmX,\bmY)  \rk+\e_{\bmY,\bmY^\prime}\lk  k_m(\bmY,\bmY^\prime)  \rk \\
            &= \int f_\bmtheta(\bm w) P^{\bmW}_m (\rmd \bm w)=\frac{1}{m}\sum_{l=1}^m f_\bmtheta(\bmW_i),
        \end{align*}
        where $P^{\bmW}_m:=\frac{1}{m}\sum_{l=1}^m \delta_{\bmW_l}$ denotes the empirical measure of $\lc\bmW_l\rc_{1\leq l\leq m}$ and $\bmX,\bmX^\prime \sim P$ and $\bmY,\bmY^\prime\sim P_\bmtheta$ are independent. Thus, $\mmd_{k_m}(P_\bmtheta,P)^2=\int f_\bmtheta(\bm w) P^{\bmW}_m (\rmd \bm w)$ with 
        $$\e\lk \int f_\bmtheta(\bm w) P^{\bmW}_m (\rmd \bm w)\rk=\e\lk f_\bmtheta(\bmW)\rk=\mmd_k(P_\bmtheta,P)^2.$$ 
        
        Next, we observe that $\vert f_\bmtheta \vert \leq 4$ for all $\bmtheta\in\Theta$ and define the function class 
        $$\mathcal{F}:=\{  f=8^{-1}\lc f_\bmtheta+4\rc \text{ or } f=8^{-1}\lc 4-f_\bmtheta\rc\text{ for some } \bmtheta \in\Theta \},$$
        which solely contains functions mapping to $[0,1]$. Thus, for every $\tau>0$, we can use the Rademacher complexity bounds from \cite[Theorem 8]{bartlett2002rademacher} to obtain that with probability at least $1-\tau$ uniformly over all $f\in\mathcal{F}$ 
        \begin{align*}
              \e\lk f(\bmW)\rk \leq & \frac{1}{m}\sum_{l=1}^m f(\bmW_l)+ \e_{\bmW}\lk \e\lk \sup_{f\in\mathcal{F}} \bigg\vert  \frac{2}{m}\sum_{i=1}^m \sigma_i f(\bmW_l) \bigg\vert\ \Bigg\vert \lc \bmW_l\rc_{1\leq l\leq m} \rk \rk \\
              &+\sqrt{\frac{8\log(2/\tau) }{m}},
        \end{align*}
        where $\lc \sigma_i\rc_{1\leq i\leq m}$ denote i.i.d.\ Rademacher variables. Therefore, multiplying with $8$, we obtain
        \begin{align*}
            \sup_{\bmtheta\in\Theta} \bigg\vert \e\lk f_\bmtheta(\bmW)\rk-\frac{1}{m}\sum_{l=1}^m f_\bmtheta(\bmW_l) \bigg\vert \leq &  8\e_{\bmW}\lk \e\lk \sup_{f\in\mathcal{F}} \bigg\vert  \frac{2}{m}\sum_{i=1}^m \sigma_i f(\bmW_l) \bigg\vert\ \Bigg\vert \lc \bmW_l\rc_{1\leq l\leq m} \rk \rk\\
            &+8\sqrt{\frac{8\log(2/\tau) }{m}} ,
        \end{align*} 
        since $8^{-1}(f_\bmtheta+4)$ as well $8^{-1}(-f_\bmtheta+4)$ are contained in $\mathcal{F}$. Moreover, 
        $$ \sup_{\bmtheta\in\Theta} \bigg\vert  \e\lk f_\bmtheta(\bmW)\rk-\frac{1}{m}\sum_{l=1}^m f_\bmtheta(\bmW_l) \bigg\vert =\sup_{\bmtheta\in\Theta} \big\vert \textbf{}\mmd_k(P_\bmtheta,P)- \mmd_{k_m}(P_\bmtheta,P)\big\vert . $$
        Thus, it remains to bound the Rademacher complexity 
        \begin{align*}
            &8\e_{\bmW}\lk \e\lk \sup_{f\in\mathcal{F}} \bigg\vert  \frac{2}{m}\sum_{i=1}^m \sigma_i f(\bmW_l) \bigg\vert\ \Bigg\vert \lc \bmW_l\rc_{1\leq l\leq m} \rk \rk\\
            &=\e_{\bmW}\lk \e\lk \sup_{\bmtheta\in\Theta}\sup_{a\in\{-1,1\}} \bigg\vert  \frac{2}{m}\sum_{i=1}^m \sigma_i\lc  af_\bmtheta(\bmW_l)+4 \rc\bigg\vert\ \Bigg\vert \lc \bmW_l\rc_{1\leq l\leq m} \rk \rk,
        \end{align*}
        We use \cite[Theorem 12.5]{bartlett2002rademacher} to bound
        \begin{align*}
          &\e_{\bmW}\lk \e\lk \sup_{\bmtheta\in\Theta}\sup_{a\in\{-1,1\}} \bigg\vert  \frac{2}{m}\sum_{i=1}^m \sigma_i\lc a f_\bmtheta(\bmW_l)+4 \rc\bigg\vert\ \Bigg\vert \lc \bmW_l\rc_{1\leq l\leq m} \rk \rk\\
          &\leq  2\e_{\bmW}\lk \e\lk \sup_{\bmtheta\in\Theta} \bigg\vert  \frac{1}{m}\sum_{i=1}^m \sigma_i  f_\bmtheta(\bmW_l) \bigg\vert\ \Bigg\vert \lc \bmW_l\rc_{1\leq l\leq m} \rk \rk  +8/\sqrt{m} \\
            &=:2 \e_{\bmW}\lk  \hat{R}_m(\Theta)\rk+8/\sqrt{m},
        \end{align*}
        with 
        $$\hat{R}_m(\Theta):=\e\lk \sup_{\bmtheta\in\Theta} \bigg\vert  \frac{1}{m}\sum_{i=1}^m \sigma_i  f_\bmtheta(\bmW_l) \bigg\vert\ \Bigg\vert \lc \bmW_l\rc_{1\leq l\leq m} \rk$$
        denoting the empirical Rademacher complexity of $\mathcal{F}_\Theta$. Thus, we can apply Dudley's theorem \cite[Theorem 3.1]{koltchinskii2011} to obtain
        $$ \hat{R}_m(\Theta)\leq 12m^{-1/2}\int_0^\infty\sqrt{ N\lc \epsilon,\mathcal{F}_\Theta,\Vert \cdot\Vert_{P_m^W} \rc } \rmd\epsilon \leq \frac{12f(m)}{m^{1/2}}, $$
        as $\vert f_\theta\vert\leq 4$.
        Combining the above we obtain with probability at least $1-\tau$
        $$ \sup_{\bmtheta\in\Theta} \bigg\vert \e\lk f_\bmtheta(\bmW)\rk-\frac{1}{m}\sum_{l=1}^m f_\bmtheta(\bmW_l) \bigg\vert \leq  \frac{24f(m)+8\lc 1+\sqrt{8\log(2/\tau)}\rc }{\sqrt{m}}.$$
   \end{enumerate}     
\end{proof}

\section{Additional plots}
\label{appaddplots}

\begin{figure}
\centering
    \subfloat[Gaussian mixture distribution $2$-dim with $5$ mixture components]{
    \includegraphics[scale=0.45]{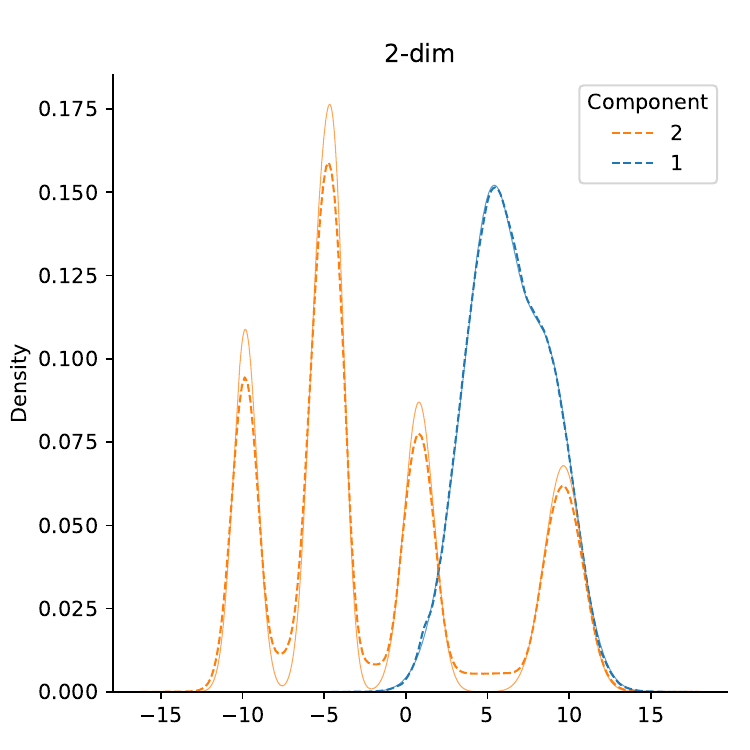}
    \includegraphics[scale=0.45]{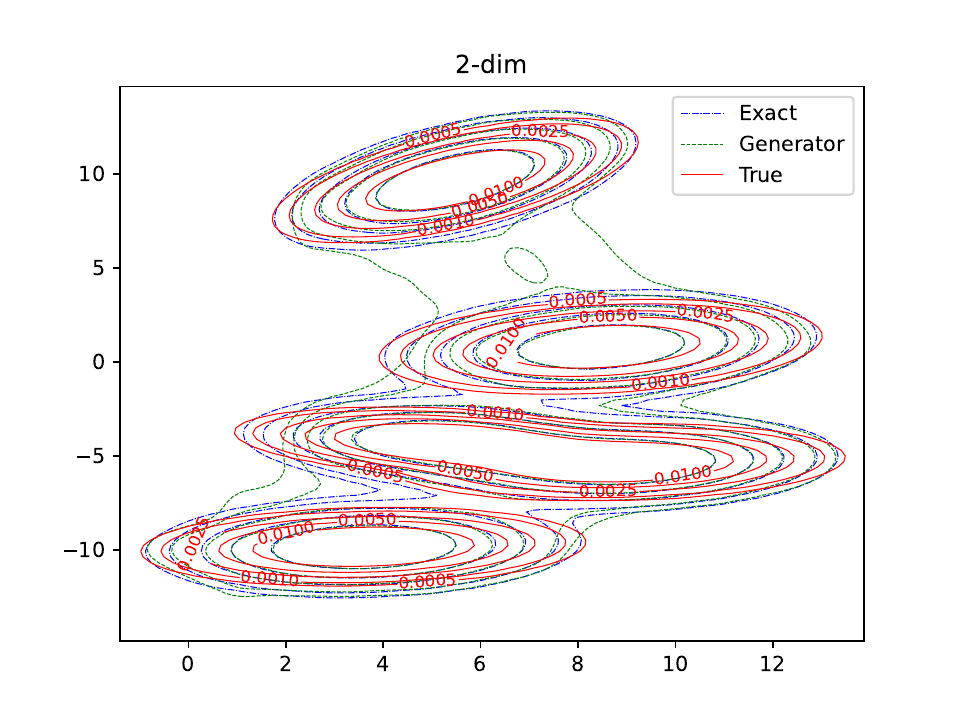}  
    }\vfill  
    \subfloat[Gaussian mixture distribution $2$-dim with $10$ mixture components]{
        \includegraphics[scale=0.45]{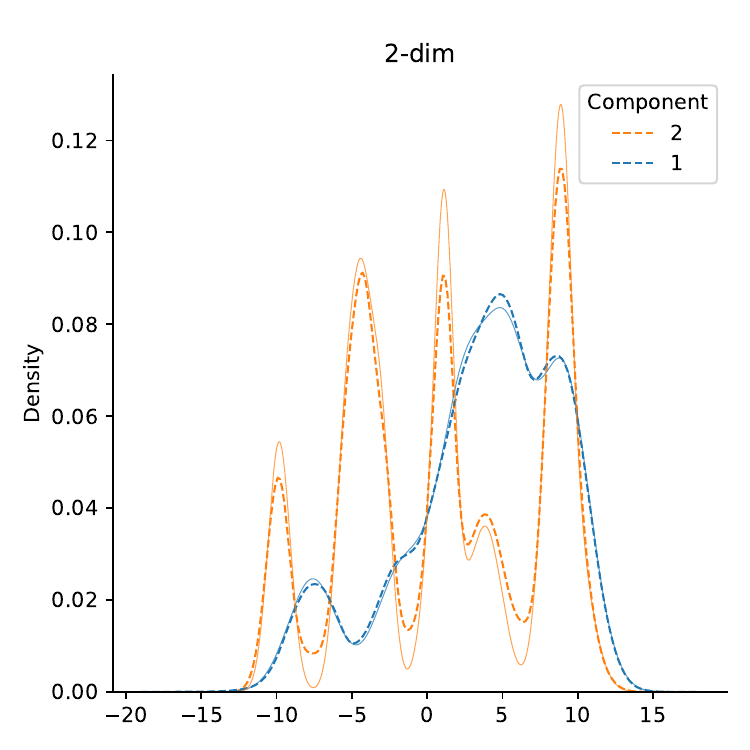}
    \includegraphics[scale=0.45]{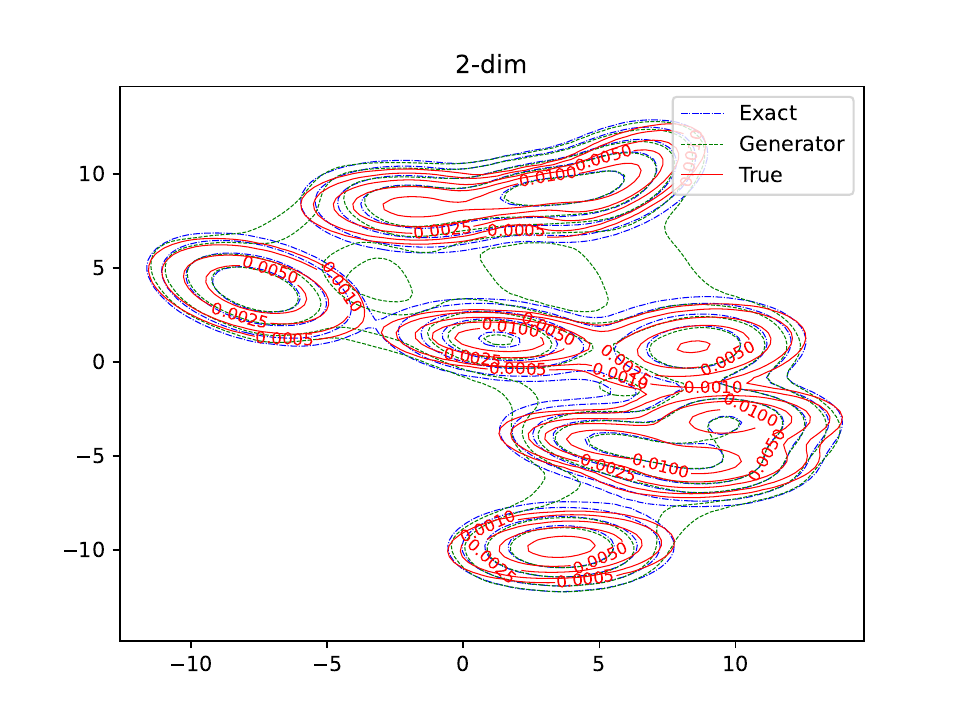}  
    }\vfill 
    \subfloat[Gaussian mixture distribution $10$-dim with $5$ mixture components]{
        \includegraphics[scale=0.45]{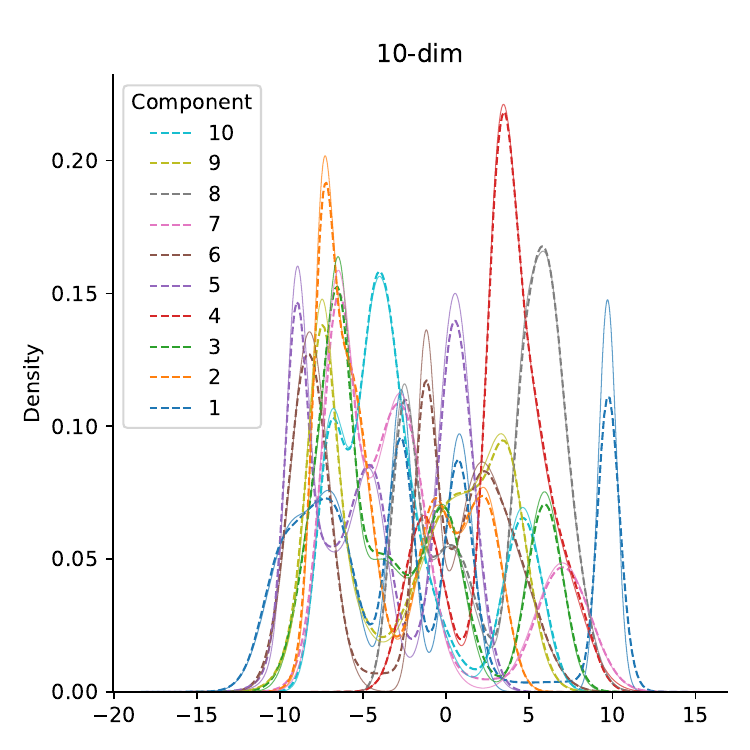}
    \includegraphics[scale=0.45]{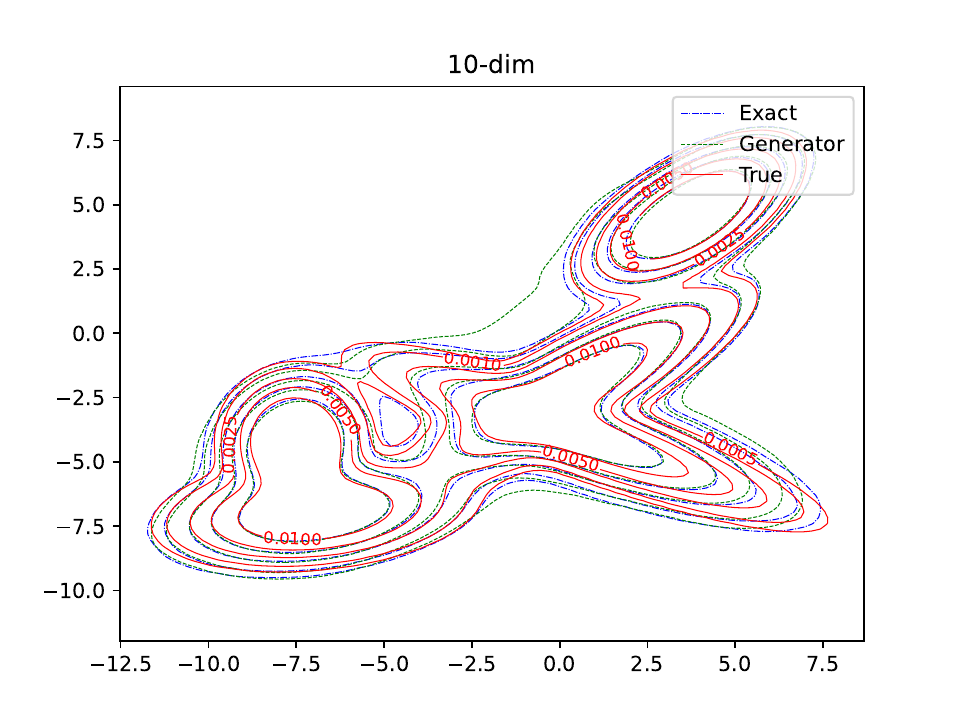}  
    }\vfill 
    \caption{Estimated marginal (left) and bivariate (right) densities of the Gaussian mixture distribution in dimensions $2,5,10$ with $2,5,10$ mixture components. The bivariate densities correspond to the last two components of the corresponding random vector. The solid red lines show the true densities and the dashed green and dash-dotted blue lines correspond to estimated densities from the generator and the exact simulation algorithm.}
\label{figgaussadd}
\end{figure}

\end{document}